\documentclass[journal]{IEEEtran}

\usepackage{booktabs}
\usepackage{mathrsfs}
\usepackage{textcomp, graphicx, subfigure}
\usepackage{epsfig}
\usepackage{amssymb}
\usepackage{amsmath}
\usepackage{amsfonts}
\usepackage{algorithmic}
\usepackage{algorithm}
\usepackage{multirow}
\usepackage{cite}
\usepackage[dvipsnames]{xcolor}
\usepackage{mathtools,bbm}
\usepackage{enumerate}
\usepackage{lipsum}
\usepackage{mathtools}
\usepackage{cuted}
\usepackage{cancel}
\usepackage{verbatim} 
\usepackage{mathtools}
 
 \usepackage{scalerel,amssymb}

\usepackage[utf8]{inputenc}
\usepackage[english]{babel}

\usepackage{amsthm}
\newtheorem{theorem}{Theorem}[section]
\usepackage{mathtools}
\DeclarePairedDelimiter\ceil{\lceil}{\rceil}
\begin{document}



\title{SF-SGL:  \underline{S}olver-\underline{F}ree \underline{S}pectral \underline{G}raph \underline{L}earning from Linear Measurements
}
\author{\IEEEauthorblockN{Ying Zhang$^*$, Zhiqiang Zhao$^*$ and Zhuo Feng}\\
\IEEEauthorblockA{\textit{Department of Electrical and Computer Engineering} \\
\textit{Stevens Institute of Technology}\\
Hoboken, NJ, USA \\
\{{yzhan232, zzhao76, zhuo.feng}\}@stevens.edu}

}
\maketitle
\def\thefootnote{*}\footnotetext{These authors contributed equally to this work}\def\thefootnote{\arabic{footnote}}
\begin{abstract}
This work  introduces a highly-scalable spectral graph densification framework (SGL) for learning resistor networks with linear measurements, such as node voltages and currents. We show that the proposed graph learning approach is equivalent to solving the classical graphical Lasso problems with Laplacian-like precision matrices.  We prove that given $O(\log N)$ pairs of voltage and current  measurements, it is possible to recover  sparse  $N$-node resistor networks that  can well preserve the effective resistance distances on the original graph. In addition, the learned graphs also preserve the structural  (spectral) properties of the original graph, which  can potentially be leveraged in many circuit design and optimization tasks. 

To achieve more scalable performance,   we also introduce a solver-free method (SF-SGL) that  exploits multilevel spectral approximation of the graphs and allows for a scalable and flexible decomposition of the entire graph spectrum (to be learned) into multiple different eigenvalue clusters (frequency bands). Such a solver-free approach allows us to more efficiently identify the most spectrally-critical edges for reducing various ranges of spectral embedding distortions. 
Through extensive experiments for a variety of real-world test cases, we show that  the proposed approach is highly scalable for learning  sparse resistor networks without sacrificing solution quality.  We  also introduce a data-driven EDA algorithm for vectorless power/thermal integrity verifications to allow estimating worst-case voltage/temperature (gradient) distributions across the entire chip by leveraging a  few voltage/temperature measurements. 
\end{abstract}







\section{Introduction}\label{sect:introduction}
Recent years have witnessed a surge of interest in machine learning on graphs \cite{hamilton2017representation}, with the goal of encoding high-dimensional data associated with nodes, edges, or (sub)graphs into low-dimensional vector representations that well preserve the original graph structural (manifold) information. Graph learning  techniques have shown promising results for various important applications such as vertex (data) classification \cite{kipf2017iclr,grover2016node2vec}, link prediction (recommendation systems) \cite{ying-gcn-kdd2018, zhang2018link}, community detection \cite{ zhou2018graph, cai2018comprehensive, goyal2018graph}, drug discovery \cite{rathi2019practical,lim2019predicting }, solving partial differential equations (PDEs) \cite{belbute2020combining, li2020multipole,iakovlev2020learning}, and electronic design automation (EDA) \cite{ma-gnn-test-dac19, zhang2019circuit, wang2020gcn,mirhoseini2021graph}.

 Modern graph learning tasks (without a known input graph topology) typically involve the following two key tasks: (1) graph topology learning for converting high-dimensional node feature (attribute) data into a  graph representation, and (2) graph embedding for converting graph-structured data (e.g. graph topology and node features) into low-dimensional vector representations to facilitate downstream machine learning or data mining tasks. 
 
 Although there exist abundant research studies on graph embedding techniques \cite{hamilton2017representation,hamilton2017inductive,deng2019graphzoom}, it still remains challenging to  learn a meaningful graph topology from  a given data set. To this end, the well-known graphical Lasso method has been proposed as a sparse penalized maximum likelihood estimator for the concentration or precision matrix (inverse of covariance matrix) of a multivariate elliptical distribution \cite{friedman2008sparse}.  The latest  graph signal processing (GSP) based learning techniques also have been proposed to estimate sparse graph Laplacians, which has shown very promising results     \cite{egilmez2017graph,dong2019learning}. For example,        the graph topology learning problem is addressed by restricting the precision matrix (inverse of the sample covariance matrix)  to be a graph Laplacian-like matrix and maximizing a posterior estimation of {{attractive Gaussian Markov Random Field (GMRF)}},
  while an $\ell_1$-regularization term  is leveraged to promote graph sparsity \cite{egilmez2017graph};
      a graph Laplacian learning method is proposed by imposing   additional  spectral
 constraints \cite{kumar2019structured};
  a   graph topology learning approach (GLAD)  is introduced based on an Alternating Minimization (AM) algorithm \cite{shrivastava2019glad}, while \cite{pu2021learning} tries to learn a mapping from node data to the graph structure based on the idea of learning to optimise (L2O) \cite{li2016learning,chen2021learning}.

However, {the state-of-the-art  graph topology learning methods \cite{ egilmez2017graph, dong2019learning, shrivastava2019glad} do not scale  to large data sets} due to their high algorithm complexity. For  example,  recent graph topology learning methods based on  Laplacian matrix estimation    require solving convex   optimization problems, which  have a time complexity of at least {$O(N^2)$} per iteration for $N$ data points and thus can only be applied to rather small data sets  with  a few hundreds of data points \cite{ egilmez2017graph,  dong2019learning,shrivastava2019glad}; the state-of-the-art deep learning based approach (GLAD) can only handle a few thousands of nodes on a high-performance  GPU \cite{shrivastava2019glad}. Consequently, existing graph topology learning methods can not be efficiently applied in   electronic design automation (EDA) tasks considering   the sheer size of modern integrated circuit systems that typically involve millions or billions of elements.
  \begin{figure}
  \includegraphics[width=0.9995\linewidth]{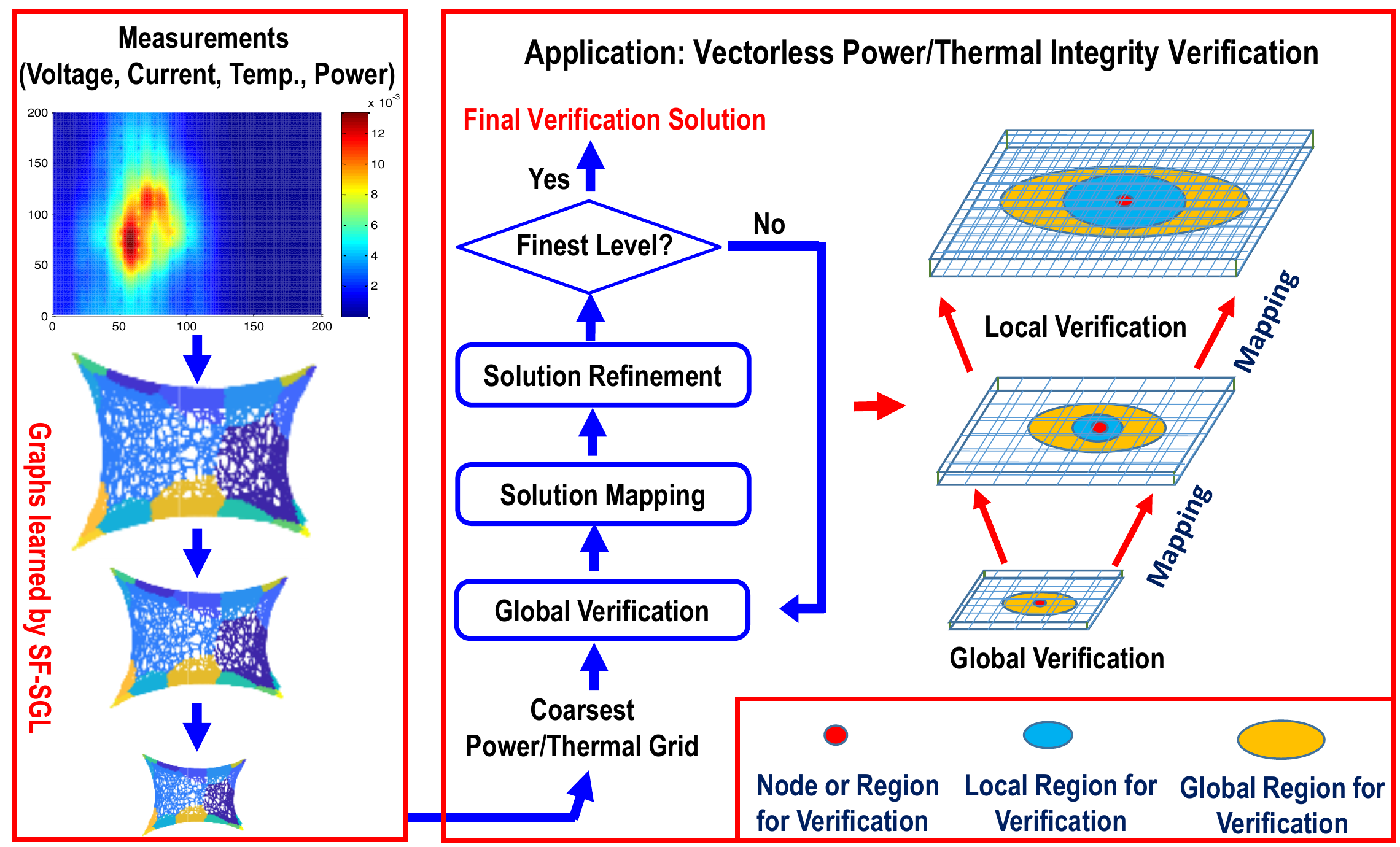}
  \caption{Spectral graph learning for data-driven vectorless  verification.}
  \label{fig:overview_sfsgl}
\end{figure}


 To the best of our knowledge, this paper introduces the first scalable spectral method (SF-SGL) for learning resistor networks from linear voltage and current measurements based on our recent SGL framework \cite{feng2021sgl}. An indispensable  step in the original SGL framework requires to compute Laplacian eigenvalues/eigenvectors for estimating spectral  embedding distortions \cite{feng2021sgl}, which can be potentially accelerated by leveraging the latest theoretical results in spectral graph theory. However, the state-of-the-art spectral  algorithms  strongly rely on fast Laplacian  solvers that are usually difficult to implement in practice and inherently-challenging to accelerate on parallel processors. For example, the effective-resistance based spectral sparsification method \cite{spielman2011graph} requires multiple Laplacian matrix solutions for computing each edge's leverage (sampling) score, while the latest spectral-perturbation based algorithm \cite{feng2020grass} leverages a  Laplacian solver  for   estimating each edge's spectral importance. The proposed algorithm (SF-SGL) is based on {a scalable multilevel spectral graph  densification framework}  for   estimating   attractive Gaussian Markov Random Fields (GMRFs).
 SF-SGL  can efficiently solve  the graphical Lasso problem \cite{friedman2008sparse} with a Laplacian-like precision matrix  by iteratively including the most influential edges to dramatically reduce  spectral embedding distortions. A    unique property of the learned graphs  is that  {the spectral embedding or effective-resistance  distances on the constructed graph  will encode the similarities} between the original input data points (node voltage measurements). Our method  allows each iteration to be completed in $O(N \log N)$ time, whereas existing state-of-the-art methods \cite{ egilmez2017graph,dong2019learning}  require at least $O(N^2)$ time for each iteration. Our analysis for sample complexity   shows that by leveraging the proposed spectral algorithms it is possible to  accurately estimate a sparse resistor network with only $O(\log N)$ voltage (and current) measurements (vectors). 


The proposed solver-free approach (SF-SGL) will also bring new opportunities for developing brand new physics-informed, data-driven EDA algorithms and applications, as shown in Figure \ref{fig:overview_sfsgl}.
In this work, by exploiting   SF-SGL we introduce a data-driven EDA algorithm for vectorless power/thermal integrity verifications to allow estimating worst-case voltage/temperature (gradient) distributions across the entire chip by leveraging only a  few voltage/temperature measurements that can be potentially obtained from on-chip voltage/temperature sensors \cite{chung2011autocalibrated,anik2020chip,ku2019voltage}. 

{{
The  main contribution of this work has been summarized as follows:
\begin{enumerate}
  \item This work introduces a spectral graph densification framework (SGL) for learning resistor networks with linear measurements. We prove that given $O(\log N)$ pairs of voltage and current measurements, it is possible to recover sparse  $N$-node resistor networks that can well preserve the effective resistance distances on the original graph.
 \item To achieve more scalable performance, a solver-free spectral graph learning framework (SF-SGL) is proposed that utilizes the multilevel spectral graph densification framework for constructing the learned graph. Compared to the previous SGL method which requires the Laplacian solver for estimating spectral embedding distortions, the proposed SF-SGL allows us to more efficiently identify the critical edges for constructing the learned graphs.
\item Our extensive experimental results show that the proposed method can produce a hierarchy of high-quality learned graphs in nearly-linear time for a variety of real-world, large-scale graphs and circuit networks.  When compared with prior state-of-the-art spectral methods, such as SGL \cite{feng2021sgl}, the proposed SF-SGL can construct the learned graph in a much faster way with better graph quality.
\item The proposed method has been validated with the application of the data-driven EDA algorithm for vectorless power grid and thermal integrity verification for estimating worst-case voltage/temperature distributions of the entire chip, showing reliable verification accuracy and efficiency.
\end{enumerate}
}}

The rest of this paper is organized as follows:  Section \ref{background_sec} introduces the background of graph learning techniques. Section \ref{main_sec} introduces the theoretical foundation of the original single-level  spectral graph learning  framework (SGL), which also includes the sample and algorithm complexity analysis. Section \ref{sec:solver_free}  extends the SGL framework by introducing a more scalable solver-free  multilevel graph learning approach (SF-SGL). Section
\ref{result_sec}  demonstrates extensive experimental results for learning a variety of real-world, large-scale graph problems, as well as data-driven  vectorless integrity verification tasks, which is followed by the conclusion of this work in Section \ref{conclusion}.

\section{Background}
\label{background_sec}
\subsection{Introduction to Graph Topology Learning}
Given $M$ observations on $N$ data entities  in a data matrix $X=[x_1, ..., x_M]\in {\mathbb{R} ^{N\times M}}$, each column vector of $X$ can be considered as a signal on a graph. {For example, the MNIST data set \cite{lecun1998gradient},  which includes $70,000$ images of handwritten digits with each image having $784$ pixels, will result in a feature matrix $X\in {\mathbb{R} ^{N\times M}}$ with $N=70,000$ and $M=784$.}   The recent  GSP-based graph learning methods \cite{dong2016learning,kalofolias2016learn, egilmez2017graph,kalofolias2017large, dong2019learning} estimate   graph Laplacians from $X$ for achieving the following two desired characteristics: 
\begin{enumerate}
\item \textbf{Smoothness of  graph signals.} The graph signals corresponding to the real-world data should  be sufficiently smooth on the learned graph structure: the signal values will only change gradually across connected neighboring nodes.
\item \textbf{Sparsity of  estimated Laplacians.} Graph sparsity   is another critical consideration in graph learning. One of the most important motivations of learning a graph is to use it for downstream computing tasks. Therefore,  more desired graph topology learning algorithms should allow better capturing and  understanding the  global structure (manifold) of the data set, while producing sufficiently sparse graphs  that   can be easily stored and efficiently manipulated in the downstream  algorithms, such as circuit simulations/optimizations, network partitioning, dimensionality reduction, data visualization, etc.
\end{enumerate}

 \subsection{Existing Methods for  Graph Topology Learning}\label{sec:method}
 \textbf{Problem formulation.} Consider a random vector $x\sim N(0,\Sigma)$ with probability density function:
  \begin{equation}
      f(x)=\frac{\exp{\left(-\frac{1}{2}x^\top \Sigma^{-1}x\right)}}{(2\pi)^{N/2}\det(\Sigma)^{(1/2)}}\propto \det(\Theta)^{1/2}\exp{\left (-\frac{1}{2}x^\top \Theta x\right )},
  \end{equation}
  where $\Sigma=\mathbb{E}[xx^\top]\succ 0$ denotes the covariance matrix, and $\Theta=\Sigma^{-1}$ denotes the precision matrix (inverse covariance matrix). Prior graph topology learning methods aim at estimating sparse precision matrix $\Theta$ from potentially high-dimensional input data, which fall into the following two categories:
 
 \underline{\textbf{(A) The graphical Lasso}} method aims at estimating a sparse precision matrix $\Theta$ using  convex optimization to maximize the log-likelihood   of $f(x)$ \cite{friedman2008sparse}:  
\begin{equation}\label{formula_lasso}
\max_{\Theta}:  \log\det(\Theta)-  Tr(\Theta S)-{\beta}{{\|\Theta\|}}^{}_{1},
\end{equation}
where $\Theta$ denotes a  non-negative definite   precision matrix,   $S$ denotes a sample covariance matrix, and   $\beta$ denotes a regularization parameter. 
The first two terms together can be interpreted as the log-likelihood under a GMRF.  $\|\bullet\|_1$  denotes the entry-wise $\ell_1$ norm, so ${\beta}{{\|\Theta\|}}^{}_{1} $ becomes the sparsity promoting regularization term. This model learns the graph structure by maximizing the penalized log-likelihood. 
{When the sample covariance matrix $S$ is obtained from $M$ i.i.d (independent and identically distributed)  samples $X=[x_1,...,x_M]$, where $X\sim N(0, S)$ has an $N$-dimensional Gaussian distribution with zero mean}, each element in the precision matrix $\Theta_{i,j}$ encodes the conditional dependence between variables $X_i$ and $X_j$. For example, $\Theta_{i,j}=0$ implies that  variables $X_i$ and $X_j$
are conditionally independent, given the rest.   

{\underline{\textbf{(B) The GSP-based Laplacian estimation}} methods have been recently introduced for more efficiently solving the following  convex  problem \cite{dong2019learning,lake2010discovering}:}
\begin{equation}\label{opt2}
  {\max_{\Theta}}: F(\Theta)= \log\det(\Theta)-  \frac{1}{M}Tr({X^\top \Theta X})-\beta {{\|\Theta\|}}^{}_{1},
\end{equation}
where   {$\Theta={L}+\frac{1}{\sigma^2}I$}, ${L}$ represents the set of valid   Laplacian matrices, where $Tr(\bullet)$ denotes the matrix trace, $I$ denotes the identity matrix, and $\sigma^2>0$ represents the prior feature variance. The three terms in   (\ref{opt2})  correspond to the terms  $\log\det(\Theta)$, $Tr(\Theta S)$ and  ${\beta}{{\|\Theta\|}}^{}_{1}$ in   (\ref{formula_lasso}),  respectively.  When  every   column vector  in the  data matrix $X$ 
is regarded as a graph signal vector, there is a close connection between  the formulation in (\ref{opt2}) and the original graphical Lasso problem \cite{friedman2008sparse}. 
Since {$\Theta={L}+\frac{1}{\sigma^2}I$}   are  symmetric and positive definite  (PSD) matrices (or M matrices) with non-positive off-diagonal entries, this formulation will lead to the estimation of  attractive GMRFs \cite{ dong2019learning, slawski2015estimation}. 
{When  $X$ is non-Gaussian}, the formulation in (\ref{opt2}) can be regarded as a
Laplacian estimation method by minimizing the Bregman divergence between positive definite matrices   induced by the function $\Theta \mapsto -\log \det(\Theta)$ \cite{slawski2015estimation}.


 \section{SGL: A Spectral  Learning Approach}\label{main_sec}

  \begin{figure}
\centering
\epsfig{file=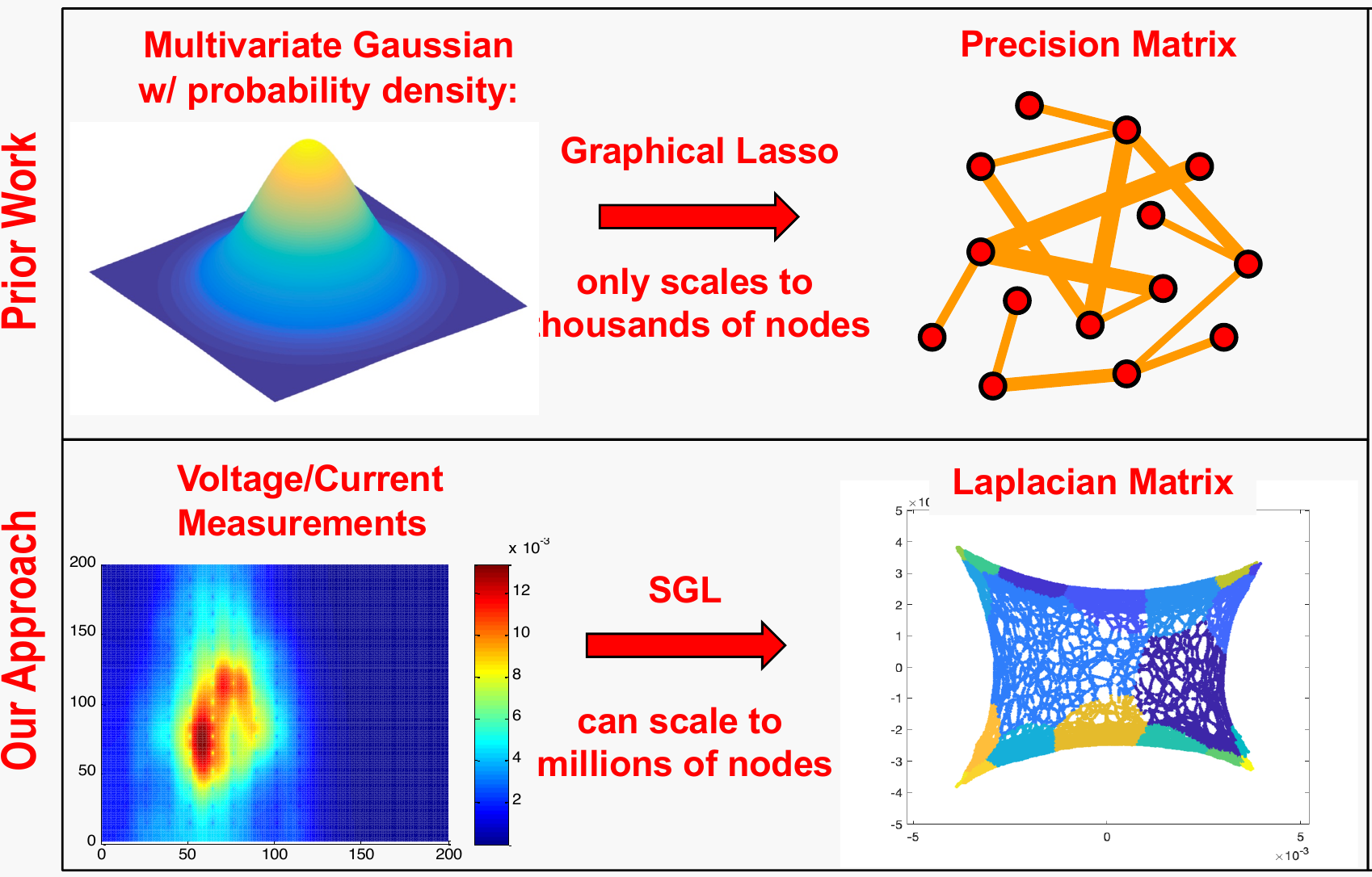, width=0.4965\textwidth}
\caption{Overview of the proposed framework for graph topology learning. \protect\label{fig:sgloverview}}
\end{figure}
 Consider $M$ linear measurements of  $N$-dimensional voltage and current vectors   stored in  data matrices $X\in {\mathbb{R} ^{N\times M}}$ and $Y\in {\mathbb{R} ^{N\times M}}$, where the $i$-th column vector $X(:,i)$ is  a  voltage response (graph signal) vector  corresponding to the  $i$-th input current   vector $Y(:,i)$. This work introduces a      spectral graph learning method (SGL) for     Laplacian   matrix  estimation  by   exploiting the linear voltage ($X$) and current ($Y$) measurements \cite{feng2021sgl}, as shown in Figure \ref{fig:sgloverview}.

 \begin{figure*}
\centering
\epsfig{file=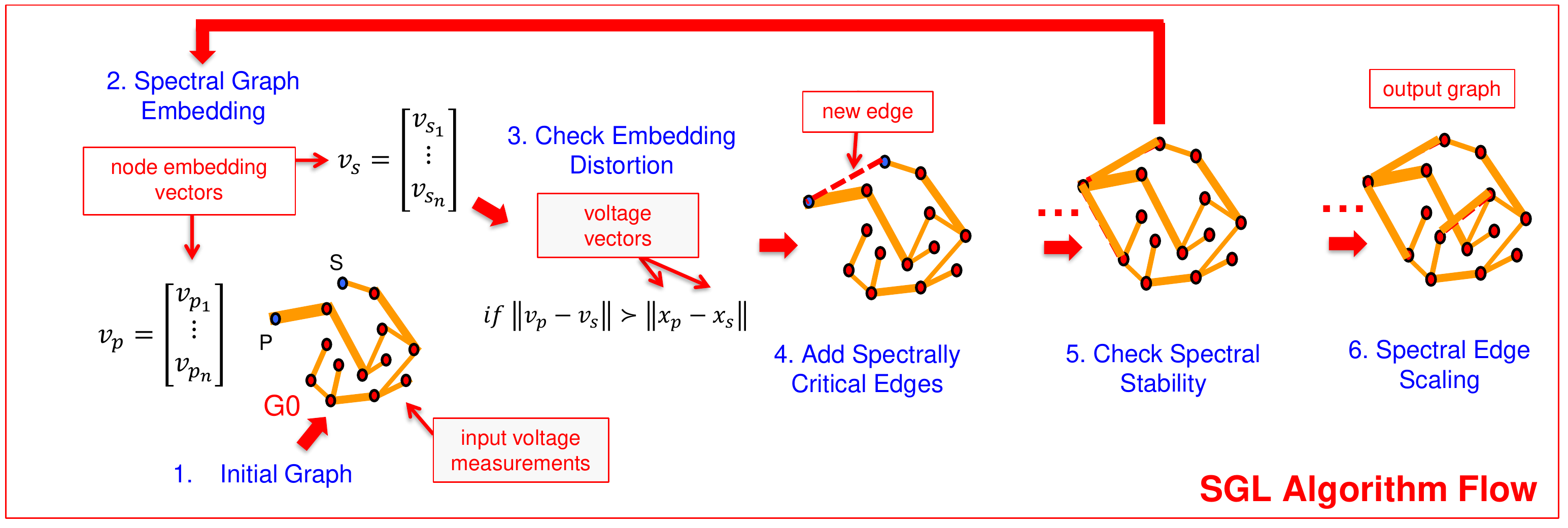, width=0.9799598\textwidth}
\caption{SGL  for spectral  graph topology learning  from voltage/current measurements. {($v_p$ and $v_s$ are node embedding vectors for node $p$ and $s$, $x_p$ and $x_s$ are voltage vectors for node $p$ and $s$.)} \protect\label{fig:sgl}}
\end{figure*}


To quantify the smoothness of a graph signal vector $x$ over a   undirected graph   $G=(V,E, w)$,  the following Laplacian quadratic form  can been adopted:
\begin{equation}\label{quad_form}
{x^\top}L x= \sum\limits_{\left( {s,t} \right) \in E}
{{w_{s,t}}{{\left( {x\left( s \right) - x\left( t \right)}
\right)}^2}},
\end{equation}
where $L=D-W$ denotes the graph Laplacian  matrix, $w_{s,t}=W(s,t)$  denotes  the weight for edge ($s,t$), while  $D$ and $W$ denote the  degree  and   the weighted adjacency matrices, respectively. {The smoothness  of a set of signals }$X$ over graph $G$ can be computed using the following matrix trace  \cite{kalofolias2016learn}:
\begin{equation}\label{trace}
Q(X,L)=Tr({X^\top L X}).
\end{equation}

\subsection{Gradient Estimation via  Perturbation Analysis}\label{sec:theory}
Express the graph Laplacian matrix as follows
\begin{equation}\label{LaplacianEdge}
L=\sum\limits_{\left( {s,t} \right) \in E}
{w_{s,t}}e_{s,t}e^\top_{s,t},
\end{equation}
where  ${e_{s}} \in \mathbb{R}^N$ denotes  the standard basis vector with all zero entries except for the $s$-th entry being $1$, and    ${e_{s,t}}=e_s-e_t$. By substituting (\ref{LaplacianEdge}) into    (\ref{opt2}),
we have:
\begin{equation}\label{opt1}
\begin{split}
F &= \sum\limits_{i=1}^N \log (\lambda_i+\frac{1}{\sigma^2})
-  \\&\frac{1}{M}\left(\frac{Tr(X^\top X)}{\sigma^2}+\sum\limits_{\left( {s,t} \right) \in E}
{w_{s,t}}\|X^\top e_{s,t}\|_2^2\right)-4 \beta \sum\limits_{\left( {s,t} \right) \in E}
{w_{s,t}},
\end{split}
\end{equation}
{{where  ${\lambda _i}$ are non-decreasing eigenvalues of $L$ for $i=1, \cdots, N$. Given the eigenvalue ${\lambda _i}$ and the corresponding eigenvector $u_i$, it satisfies:}
\begin{equation}\label{eigen}
L u_i = \lambda_i u_i.
\end{equation}}

Taking the partial derivative of (\ref{opt1}) with respect to  $w_{s,t}$ leads to:
\begin{equation}\label{optF}
\frac{\partial F}{\partial w_{s,t} } = \sum\limits_{i=1}^N \frac{1}{\lambda_i+1/\sigma^2} \frac{\partial \lambda_i}{\partial w_{s,t} }
-  \frac{1}{M}\|X^\top e_{s,t}\|_2^2-4 \beta.
\end{equation}

  The last two terms in (\ref{optF}) both imply   constraints on graph sparsity: adding more edges will lead to a greater trace $Tr({X^\top \Theta X})$. Consequently, we can safely  choose $\beta=0$ in the rest of this work,  without impacting the ranking of   candidate edges and thus  the final solution quality of  SGL. 
\begin{theorem}\label{thm:pertub}
The spectral perturbation $\delta {\lambda _i} $ due to adding a candidate edge   $({s,t})$ into the latest graph will be: 
\begin{equation}\label{sc}
 \delta {\lambda _i}=  \delta w_{s,t}\left( {{{u_i^\top}e_{s,t}}  } \right)^2.
\end{equation} 
\end{theorem}
\begin{proof} 
Consider the following spectral perturbation analysis:
\begin{equation}\label{formula_eig_perturb1}
\left( {L + \delta L} \right)\left( {{u_i} + \delta {u_i}} \right) = \left( {{\lambda _i} + \delta {\lambda _i}} \right)\left( {{u_i} + \delta {u_i}} \right),
\end{equation}
where a small perturbation $\delta L$ due to including a new edge  $(s,t)$  is applied to $L$, which results in perturbed eigenvalues and eigenvectors  ${\lambda _i} + \delta {\lambda _i}$ and ${u_i} + \delta {u_i}$ for $i=1,...,N$, respectively. By only keeping   the first-order terms, we have:
\begin{equation}\label{formula_eig_perturb1_first_order}
 {L}\delta {u_i} + {\delta L}{u_i} = {{\lambda _i}{\delta {u_i}} + \delta {\lambda _i}}  {{u_i} }.
\end{equation}
Express $\delta u_i$ in terms of the original eigenvectors $u_i$  {{for}} $i=1,...,N$ as:
{{\begin{equation}\label{delta u_i}
{\delta {u_i}} = \sum\limits_{i = j}^{N} {{\alpha _j}{u_j}}.
\end{equation}}}
By substituting (\ref{delta u_i}) into (\ref{formula_eig_perturb1_first_order}), we have:
{{\begin{equation}\label{formula_eig_perturb1_first_order_expand}
 {L}\sum\limits_{j = 1}^{N} {{\alpha _j}{u_j}} + {\delta L}{u_i} = {{\lambda _i}\sum\limits_{j = 1}^{N} {{\alpha _j}{u_j}} + \delta {\lambda _i}}  {{u_i} }.
\end{equation}}}
After multiplying ${u_i^\top}$ to both sides of (\ref{formula_eig_perturb1_first_order_expand}) we have:
\begin{equation}\label{formula_eig_perturb1_conclusion}
 \delta {\lambda _i} = \delta w_{s,t}\left( {{{u_i^\top}e_{s,t}} } \right)^2.
\end{equation}
\end{proof}

Subsequently, we  construct the following  eigensubspace matrix for spectral graph embedding using the first $r-1$  nontrivial  Laplacian eigenvectors
\begin{equation}\label{subspace}
U_r=\left[\frac{u_2}{\sqrt {\lambda_2 +1/\sigma^2}},..., \frac{u_r}{\sqrt {\lambda_r +1/\sigma^2}}\right]. 
\end{equation}
  Theorem \ref{thm:pertub} allows each  edge's \textbf{spectral sensitivity} $s_{s,t}=\frac{\partial F}{\partial w_{s,t} }$ in  (\ref{optF}) to be written as: 
\begin{equation}\label{optF2}
\begin{split}
 s_{s,t}&=\sum\limits_{i=1}^N \frac{1}{\lambda_i+1/\sigma^2} \frac{\partial \lambda_i}{\partial w_{s,t} }-\frac{1}{M}\|X^\top e_{s,t}\|_2^2\\
 &= \|U_N^\top e_{s,t}\|_2^2
-  \frac{1}{M}\|X^\top e_{s,t}\|_2^2\\
&\approx \|U_r^\top e_{s,t}\|_2^2
-  \frac{1}{M}\|X^\top e_{s,t}\|_2^2\\
&=z_{s,t}^{emb}-\frac{1}{M}z_{s,t}^{data}, where ~~~r \ll N.
\end{split}
\end{equation}
In the above expressions, $z_{s,t}^{emb}=\|U_r^\top e_{s,t}\|^2_2$ and $z_{s,t}^{data}={\|X^\top e_{s,t}\|^2_2}{}$ denote the $\ell_2$ distances in the spectral embedding space as well as the    data (voltage measurement) vector space, respectively.  The partial derivative term $s_{s,t}$ in (\ref{optF2}) can be leveraged for solving the  optimization task in (\ref{opt2}) using gradient-based methods, such as the general stagewise algorithm for  group-structured learning \cite{tibshirani2015general}.
 
\subsection{Spectral   Densification for  Graph Topology Learning} 

\textbf{Spectrally-critical edges.} We define the {spectral embedding distortion} $\eta_{s,t}$ of an edge $(s, t)$  to be: 
 \begin{equation}\label{embedDist}
 \eta_{s,t}=M\frac{ z_{s,t}^{emb}}{z_{s,t}^{data}}.
 \end{equation}
We also call the candidate edge $(p, q)$ that has a relatively large spectral sensitivity ($s_{p,q}$) or embedding distortion {{$ {\eta_{p,q}}=M{z_{p,q}^{emb}}/{z_{p,q}^{data}}$}}  a {spectrally-critical edge}.  
 
 \textbf{The proposed SGL iterations.} Let's consider the $(k+1)$-th SGL iteration for finding a few  most spectrally-critical candidate edges with the largest spectral sensitivities (or embedding distortions) to be added into the latest graph $G_k=(V_k,E_k, w_k)$. 
     (\ref{optF2}) implies that including such   edges into the latest graph $G_k$ will more effectively improve the objective function and significantly mitigate the spectral embedding distortions.  
    
 \textbf{Spectral graph sparsification (prior work).}  Prior   research  shows that  for every undirected graph  there exists a  sparsified graph with $O (\frac{N\log N}{\epsilon^2})$ edges  that can be obtained  by sampling each edge $(p, q)$ with a probability (leverage score) $p_{e}$ proportional to its  effective resistance \cite{spielman2011graph}:
 \begin{equation}\label{equ:resist}
 p_e\propto  \frac{R^{\textit{eff}}_{p,q}}{R^{}_{p,q}}=w_{p,q}R^{\textit{eff}}_{p,q},
 \end{equation}
  where $R^{}_{p,q}=1/w_{p,q}$ represents the original  resistance and  $R^{\textit{eff}}_{p,q}$ is the edge effective resistance. In addition, the following inequality will be valid \cite{spielman2011graph}:
 \begin{equation}\label{equ:spectralspar}
 \forall x \in \mathbb{R}^N~~~ (1-\epsilon){x^\top}L x \le {x^\top}\Tilde{L} x \le (1+\epsilon){x^\top}L x,
 \end{equation}
where $L$ and $\Tilde{L}$ are the original and sparsified graph Laplacian matrices, respectively. 
 
 \textbf{Spectral graph densification.} For  any candidate edge  selected based on its spectral sensitivity during an SGL iteration, by setting its  weight as: 
  \begin{equation}\label{equ:edgeweight}
 {w_{p,q}}\propto\frac{1}{z_{p,q}^{data}},
 \end{equation}
 the spectral embedding distortion can be estimated as follows
   \begin{equation}\label{equ:embdist}
 {{\eta_{p,q}}=M\frac{z_{p,q}^{emb}}{z_{p,q}^{data}} \propto w_{p,q} R^{\textit{eff}}_{p,q},}
 \end{equation}
 which becomes the  leverage score for spectral  graph sparsification \cite{spielman2011graph}. Consequently, as opposed to spectral sparsification,  SGL   can be regarded as a {spectral graph densification} procedure that aims to identify and include   $O (N\log N)$  spectrally-critical edges with large spectral sensitivities (embedding distortions).

\textbf{Algorithm convergence.} The   optimal solution of (\ref{opt2}) can be achieved when the maximum edge sensitivity ($s_{max}$) in (\ref{optF2}) becomes zero  or equivalently when the maximum  spectral embedding distortion  ($\eta_{max}$) in (\ref{embedDist})  becomes one.  Moreover, upon the convergence of SGL iterations, the  spectral embedding (effective-resistance) distances on the learned graphs    will   encode the $\ell_2$ distances between the original data points (voltage measurements), which can be exploited in many important tasks, such as VLSI CAD, manifold learning, dimensionality reduction, and data visualization \cite{zhao:dac19,belkin2003laplacian,carey2017graph,imre2020spectrum}.

\subsection{Sample Complexity of the SGL Algorithm}
The sample complexity of SGL can be obtained by analyzing the required number of voltage vectors (measurements) for accurate graph learning. We assume $\sigma^2\rightarrow +\infty$. For the ground-truth graph   $G_*$,  we define its edge weight matrix  $W_*$ to be a diagonal matrix with $W_*(i,i)=w_i$, and  its injection matrix as:
\begin{equation}
B_*(i,p)=\begin{cases}
1 & \text{ if } p \text{ is  i-th edge's    head}\\
-1 & \text{ if } p \text{ is   i-th edge's tail} \\
0 & \text{ otherwise }. 
\end{cases}
\end{equation}  
Consequently, the Laplacian matrix of   $G_*$   can   be written as 
\begin{equation}\label{effRes2}
L_*=B_*^\top W_* B_*.
\end{equation}
Therefore, the effective resistance $R_*^\textit{eff}({s,t})$ between nodes $s$ and $t$ can be expressed as:
\begin{equation}\label{effRes2}
R_*^\textit{eff}({s,t})=e^\top_{s,t}L_*^+e_{s,t}=\|W_*^{\frac{1}{2}} B_*L_*^{+}e_{s,t}\|_2^2,
\end{equation}
where $L^+_*$ represents the Moore–Penrose pseudoinverse of $L_*$. According to the Johnson-Lindenstrauss Lemma, the effective-resistance distance for every pair of nodes satisfies \cite{spielman2011graph}:
\begin{equation}\label{effResJL}
(1-\epsilon)R_*^\textit{eff}({s,t})\le\|X^\top e_{s,t}\|_2^2 \le (1+\epsilon)R_*^\textit{eff}({s,t}),
\end{equation} 
where  the  data (voltage measurement) matrix $X\in \mathbb{R}^{N\times M}$ is created by going through the following steps:
\begin{enumerate}
\item Let $C$ be a random $\pm \frac{1}{\sqrt{M}}$ matrix of dimension $M\times |E|$, where $|E|$ denotes the number of edges and $M=24 \log \frac{N}{\epsilon^2}$ denotes the number of voltage measurements;
\item Construct   $Y=CW_*^{\frac{1}{2}}B_*$, with the $i$-th row vector denoted by $y^\top_i$; 
\item Solve $L_* x_i= y_i$ for all rows in $C$ ($1 \le i\le M$), and construct $X$ matrix using $x_i$ as its $i$-th column vector.
\end{enumerate}
Obviously, (\ref{effResJL}) implies that given  $M\ge O(\log {N}/{\epsilon^2})$ voltage vectors (measurements) obtained through the above procedure,  $(1\pm \epsilon)$-approximate effective resistances can be computed by 
\begin{equation}
\tilde R_*^\textit{eff}(s,t)=\|X^\top e_{s,t}\|_2^2
\end{equation}
for any pair of nodes $(s,t)$ in the original graph $G_*$. Consider the following close connection between  effective resistances and  spectral graph properties (such as the first few Laplacian eigenvalues/eigenvectors):
\begin{equation}\label{formula_Reff}
R^\textit{eff}_{s,t}=\|U_N^\top e_{s,t}\|^2_2,~\text{where~} U_N=\left[\frac{u_2}{\sqrt {\lambda_2}},..., \frac{u_N}{\sqrt {\lambda_N}}\right].
\end{equation}
Therefore,  using $O(\log {N})$ voltage measurements  would be sufficient for  SGL  to learn an $N$-node graph for well preserving the original graph effective-resistance distances.

\subsection{Key Steps in the SGL Algorithm}\label{sec:overview}
To achieve good runtime and memory efficiency in graph learning tasks that may involve a large number of nodes,  the proposed SGL algorithm  strives to iteratively identify and include the most influential (spectrally-critical) edges  into the latest graph until no such edges  can be found, through the following key steps, as shown in Figure \ref{fig:sgl}:

 \textbf{Step 1: Initial graph construction.}
(\ref{optF2}) implies that by iteratively identifying and adding the most influential edges (with the largest sensitivities) into the latest graph,   the graph spectral embedding (or effective-resistance) distance will encode the   $\ell_2$ distances between the  original data (voltage measurement)  vectors. 
{ The ideal pool of candidate edges should include all possible edge connections, e.g. forming a complete graph with $O(N^2)$ edges for $N$ data samples. The proposed graph learning algorithm can be better understood through the following steps: a) each initial edge weight of the complete graph is set to be a very small value (close to zero); b) the edge weight with the highest spectral sensitivity will be updated with a greater value; c) repeat step b) until no edge can be updated (no positive edge sensitivity exists). The above iterative scheme is equivalent to starting with no edge connectivity. However, considering all edge connections will require the quadratic complexity, which may lead to rather poor runtime scalability.}
To achieve a better runtime scalability,    $k$-nearest-neighbor (kNN) graph \cite{malkov2018efficient} can be  leveraged as a sparsified complete graph. {Note that when the effective-resistance distances encode the Euclidean distances between data samples, the proposed graph learning iterations will converge. This implies that when the learned graph has a tree-like structure, the effective-resistance distances will approximately match the shortest path distances and thus encode the Euclidean distances between data samples. Therefore, Euclidean distance becomes a natural choice for kNN graph construction.} 

However, choosing an optimal $k$ value (the number of nearest neighbors) for constructing  kNN graphs can still be challenging for general graph learning tasks:  choosing a too large $k$   allows well approximating the global structure of the manifold for the original data points (voltage measurements), but will  result in a rather dense graph;  choosing a too small $k$ may lead to many small isolated graphs, which may  slow down the   iterations.  
Since circuit networks   are typically   very sparse (e.g. 2D or 3D meshes) in nature, the voltage or current measurements (vectors) usually lie near  low-dimensional manifolds, which allows choosing a proper $k$ for our graph learning tasks. To achieve a good trade-off between   complexity and quality,   the initial graph will  be set up through the following steps: \textbf{(1)} Construct a connected kNN graph with a relatively small $k$ value (e.g. $5\le k \le 10$), which will suffice for approximating the global manifold corresponding to the original   measurement data; \textbf{(2)} Sparsify the kNN graph by extracting a  maximum spanning tree (MST)   that can serve as a reasonably good initial graph in practice. Later, SGL will   incrementally  improve the graph by iteratively adding the most influential off-tree edges   from  the   initial kNN graph until convergence. {As shown in Figure  \ref{fig:objknn}, using different k values will not significantly impact the final objective function value after going through the proposed graph learning iterations.}
  
 \textbf{Step 2: Spectral graph embedding.}
Spectral graph embedding   leverages the first few Laplacian eigenvectors  for mapping nodes onto low-dimensional space \cite{belkin2003laplacian}. The eigenvalue decomposition of  Laplacian matrix  is usually the computationally demanding in spectral graph embedding, especially for large graphs. To achieve good scalability,  it is possible to exploit recent fast  Laplacian eigensolvers with nearly-linear time complexity \cite{zhao2021wsdm}. 

  \textbf{Step 3: Influential edge identification.}
Given the first few Laplacian eigenvalues and eigenvectors, SGL will efficiently identify the most influential off-tree edges by looking at each candidate edge's sensitivity score  defined in (\ref{optF2}). Specifically,    each  off-tree edge  that belongs to  the kNN graph will be sorted according to its edge sensitivity. Only a few most influential edges with the largest sensitivities  will be included into the latest graph. Note that when   $r\ll N$, the following inequality  holds for any edge $(s, t)$: 
\begin{equation}\label{distBound}
 \|U_r^\top e_{s,t}\|^2_2=z_{s,t}^{emb} < \|U_N^\top e_{s,t}\|^2_2 \le R^\textit{eff}({s,t}),
\end{equation}
 which implies that the edge sensitivities ($s_{s,t}$) approximately computed using the first $r$ eigenvectors will always be   lower than the actual values. It is obvious that using more eigenvectors for spectral   embedding   will lead to more accurate estimation of edge sensitivities. For typical circuit networks, edge   sensitivities computed using only a small number (e.g. $r<10$) of   eigenvectors will be sufficiently accurate for ranking the off-tree edges.

  \textbf{Step 4: Convergence checking.}
  In this work, we examine the maximum edge sensitivities   computed by (\ref{optF2}) for checking the convergence of SGL iterations. If there exists no  off-tree edge that has a sensitivity greater than a given threshold ($s_{max}\ge tol$), the SGL iterations will be terminated. We note that choosing different  tolerance ($tol$) levels  will result in  graphs with  different densities. For example, choosing a smaller threshold   will result in more edges   to be included so that the final spectral embedding distances on the learned graph can more accurately encode the distances between the original data points (voltage measurements).

    \textbf{Step 5:  Spectral edge scaling.}
  To further improve the spectral approximation quality, we choose to globally scale up edge weights of the learned graph obtained via SGL. Given the ground truth Laplacian matrix ${L_*}$, the effective resistance between nodes $s$ and $t$ can be represented as $R_{s,t}^{\textit{eff}}= e_{s,t}^T {L^+_*}{e}_{s,t}$. If we consider the graph as a resistor network with each conductance value corresponding to each edge weight, $R_{s,t}^{\textit{eff}}$ can be regarded as the power dissipation of the resistor network when a unit current is flowing into node $s$ and out from node $t$. By relaxing the vector $\mathbf{e}_{s,t}$ with a random vector $y_i$ that is orthogonal to the all-one vector, it can be shown that matching the power dissipations between the original graph and the learned graph   will immediately lead to  improved spectral approximation of the learned graph. Given the normalized current vectors $Y = (y_1,  \cdots, y_M)$, corresponding measurement (voltage) vectors $X = (x_1, \cdots, x_M)$ and $\tilde{X} = (\tilde{x}_1, \cdots, \tilde{x}_M)$ can be computed for the original graph and the learned graph, respectively, where $x_i = {L^+_*} y_i$  and $\tilde{x}_i = L^+y_i$. To better match the structural properties of the original graph, we propose to globally scale up the edge weights in learned graph:
     \begin{equation}\label{scaling2}
w_{s,t}= w_{s,t}*\alpha'
\end{equation}
  with the following scaling factor $\alpha'$:
\begin{equation}\label{eqn:scaling}
    \alpha' = \frac{1}{M}\sum_{i=1}^{M}{\frac{y_i^\top L^+y_i}{y_i^\top {L^+_*} y_i}} = \frac{1}{M}\sum_{i=1}^{M}{\frac{y_i^\top \tilde{x}}{y_i^\top x_i}}.
\end{equation}


\subsection{Algorithm  Flow and Complexity of SGL }\label{main:complexity}
  The detailed SGL algorithm flow  has been shown in Algorithm \ref{alg:sgl}. All the aforementioned   steps in SGL can be accomplished in nearly-linear time by leveraging recent high-performance algorithms for   kNN graph construction \cite{malkov2018efficient},   spectral graph embedding   for influential edge identification \cite{zhao:dac19,zhao2021wsdm}, and fast Laplacian solver for edge scaling \cite{miller:2010focs,zhiqiang:iccad17}. Consequently, each SGL iteration can be accomplished in  nearly-linear time, whereas the state-of-the-art methods require at least $O(N^2)$ time \cite{dong2019learning}.  
  
\begin{algorithm}
 { \caption{The SGL Algorithm Flow.} \label{alg:sgl}
\textbf{Input:} The voltage measurement matrix $X \in {\mathbb{R} ^{N\times M}}$, input current measurement matrix $Y \in {\mathbb{R} ^{N\times M}}$, $k$ for initial kNN graph construction, $r$ for constructing the  projection matrix in (\ref{subspace}), the maximum edge sensitivity tolerance ($tol$),  and the edge sampling ratio ($0 < \beta \le 1$). ~~~\textbf{Output:} The learned graph   $G=(V, E, w)$.\\

\begin{algorithmic}[1]
    \STATE {Construct a   kNN graph $G_{o}=(V, E_o, w_o)$ based on $X$.}
     \STATE {Extract an MST subgraph $T$ from $G_{o}$. }
    \STATE {Assign $G=T=(V, E, w)$ as the initial graph. }
     \WHILE{$s_{max} \ge tol$}
     \STATE{Compute the projection matrix $U_r$ with (\ref{subspace}) for the latest graph $G$.}
     \STATE{Sort  off-tree edges $(s,t)\in E_o\setminus E$  according to their sensitivities computed by $s_{s,t}=\frac{\partial F}{\partial w_{s,t} }$ using (\ref{optF2}).}
     \STATE{Include  an off-tree edge $(s, t)$ into $G$ if its  $s_{s, t}>tol$ and it has been ranked among the top   $\ceil{ N{\beta}}$   edges. }
       \STATE{Record the maximum edge sensitivity $s_{max}$. }
     \ENDWHILE
     \STATE{Do spectral edge  scaling  using $X$ and $Y$ following (\ref{scaling2});}
    \STATE {Return the learned graph $G$.}
\end{algorithmic}
}
\end{algorithm}
\vspace{-0pt}

\section{SF-SGL: Solver-Free Spectral Graph Learning} {\label{sec:solver_free}}

\begin{figure*}[!htb]
\minipage{0.32451\textwidth}
  \includegraphics[width=\linewidth]{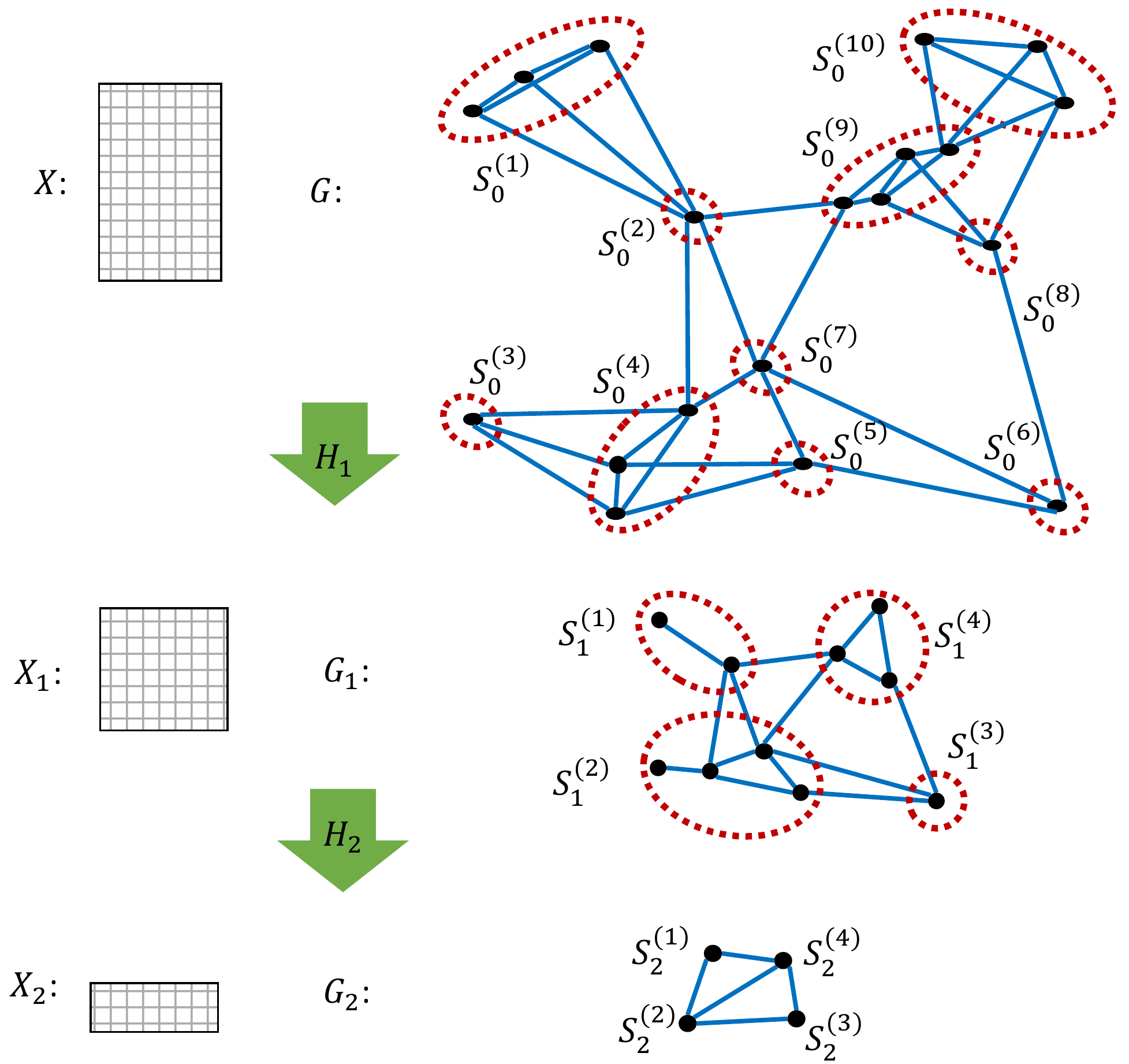}
  \caption{Graph coarsening by local   embedding.}\label{fig:coarsening}
\endminipage\hfill
\minipage{0.265\textwidth}
  \includegraphics[width=\linewidth]{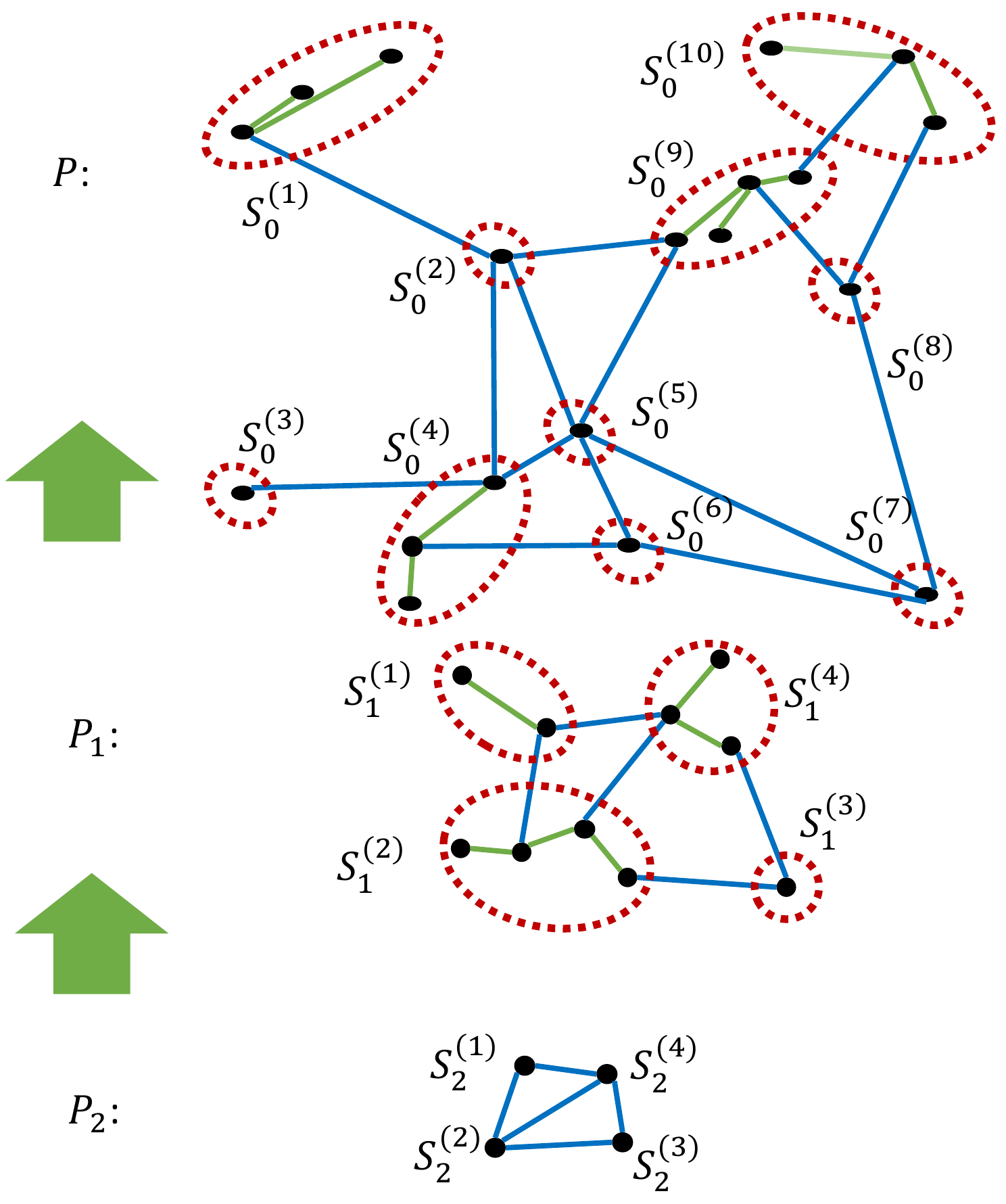}
  \caption{Graph  learning from bottom up. }\label{fig:backward}
\endminipage\hfill
\minipage{0.39074053\textwidth}%
  \includegraphics[width=\linewidth]{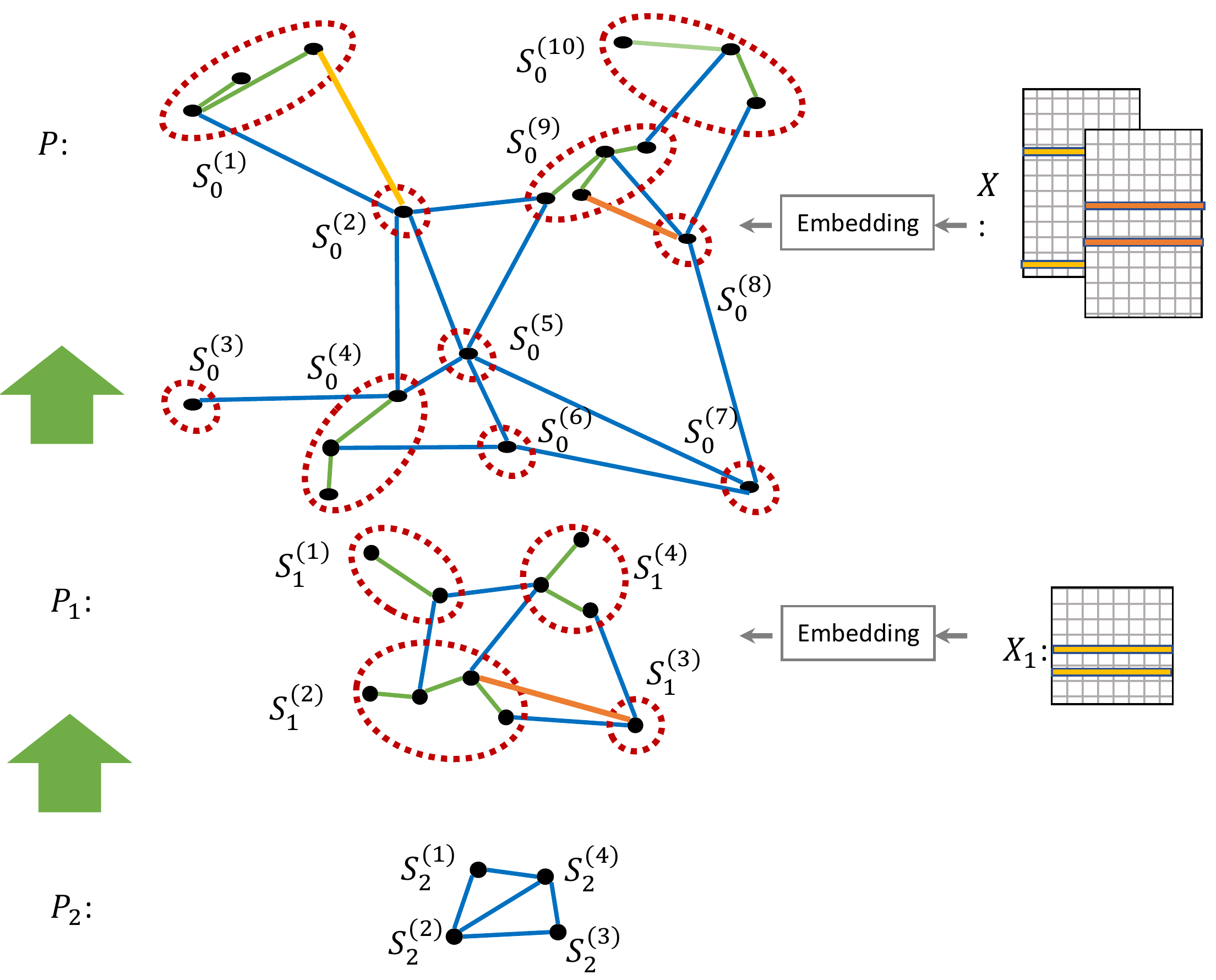}
  \caption{ Influential off-tree edge identification. }\label{fig:emebedding}
\endminipage
\end{figure*}
 

\textbf{Overview of SF-SGL.} In this work, we  extend SGL by introducing {a solver-free}, multilevel spectral graph learning scheme (SF-SGL).
The   key procedures in SF-SGL are described as follows: 
(1) given the measurement vectors, such as voltage and current matrices $X$ and $Y$,  a hierarchy of  coarse level kNN  graphs   will be constructed by exploiting a  scalable  spectral graph coarsening framework \cite{zhao2021wsdm}, as shown in Figure \ref{fig:coarsening}; (2) the SGL algorithm will be leveraged for graph learning starting from  the  coarsest level: an MST is first extracted, which is followed by a procedure for identifying spectrally-critical off-tree edges  (that belong to the coarsest kNN graph) with large embedding distortions that can be efficiently computed by exploiting a solver-free local spectral graph embedding scheme; (3) after  SGL converges at the coarsest level, the learned graph will be progressively mapped to the finer levels, as show in Figure \ref{fig:backward}, where additional candidate edges  (that belong to the coarsened kNN graphs) will be identified and included into the latest graphs with updated local  embedding vectors, as shown in Figure \ref{fig:emebedding}.  Once the learned graph is obtained, the proposed framework and learned graph can be efficiently applied in different tasks. One immediate application is to facilitate the vectorless power grid and thermal integrity verification, as shown in Figure \ref{fig:overview_sfsgl}.  

 In the rest of this paper, given an   initial kNN graph ${G}_{0}= {G}$, a series of coarsened kNN  graphs ${G}_1, {G}_2,...,{G}_{l_f}$ will be generated via the   spectral graph coarsening    \cite{zhao:dac19}, where ${G}_{l_f}$ denotes the coarsest kNN graph. The multilevel learned graphs are denoted by  ${P}_0, {P}_1,...,{P}_{l_f}$. For the sake of simplicity, all the symbols used in this paper are summarized in Table \ref{tab:symbols}.

\begin{center}
\begin {table}
\caption {Summary of symbols in the SF-SGL framework ($l=0,1,...,{l_f} , i=1,...,n_l$).} \label{tab:symbols}
\resizebox{\columnwidth}{!}{
\begin{tabular}{|c c|c c|} 
 \hline
 symbols & description & symbols & description \\  \hline
   $\Theta$ & precision matrix   & $\mathcal{C}$ & sample covariance matrix     \\ 
  ${H}_l$ & coarsening operator   & ${X}_l$ & feature matrix at level $l$    \\ 
   ${z}_{p,q}^{data}$ & data distance   & ${z}_{p,q}^{emb}$ &  embedding distance   \\ 
   ${s}_{p,q}$ & spectral sensitivity   & ${\eta}_{p,q}$ &   distortion   \\ 
 ${G}_l=({V}_l,{E}_l)$ & an undirected graph at level $l$ & ${P}_l=(V_l,E_l)$ & the learned graph\\ 
 ${\omega_{l}(p,q)}$ &  weight of edge $(p, q)$ for ${G}_l$ & $\tilde{\omega}_l(p,q)$ & weight of edge $(p, q)$ for ${P}_l$\\
 
 ${E}_l$ & edge set of ${G}_l$  & $\tilde{E}_l$ &  edge set of ${P}_l$\\
  $m_l=|{E}_l|$ & number of edges in ${G}_l$ & $\tilde{m}_l=|\tilde{E}_l|$ & number of edges in ${P}_l$ \\
 
${V}_l$ & node set at level $l$ & ${n}_l=|{V}_l|$ & number of nodes \\
 
 ${L}_{{G}_l}$ & Laplacian  of  graph ${G}_l$ &${L}_{P_l}$ & Laplacian of  graph ${P}_l$\\
  $u_{l}^{(i)}$ and $\tilde{u}_{l}^{(i)}$ & eigenvectors of  $L_{G_l}$ and $ L_{P_l}$&
 $\lambda_{l}^{(i)}$ and $\tilde{\lambda}_{l}^{(i)}$ & eigenvalues of of $L_{G_l}$ and $L_{P_l}$ \\
 \hline
\end{tabular}}
\end{table}
\end{center}



 
\subsection{Spectral Graph Coarsening
}\label{sec:sfgrass}
 \textbf{Node aggregation sets for graph coarsening.} {{As shown in Figure \ref{fig:coarsening},  coarsening the graph from $G_{l-1}$ to $G_l$ requires to cluster the node set $V_{l-1}$ of $G_{l-1}$  into $n_l$ different node aggregation sets $S^{(i)}_{l-1}$ with $i=1,..., n_l$. The subgraph $G_{l-1}[S^{(i)}_{l-1}]$ induced by each node aggregation set $S^{(i)}_{l-1}$ will be a strongly-connected component in graph $G_{l-1}$, which will be aggregated into a single node in the coarser graph $G_l$. Consequently, once the node aggregation sets are determined, a node-mapping matrix $H_{l}$ can be uniquely constructed accordingly, which   allows creating $G_{l}$ as follows:
\begin{equation}
 L_{G_{l}}:=H_{l}^{\mp}L_{G_{l-1}}H_{l}^{+},  \quad    x_{l}:=H_{l}x_{l-1},
\end{equation}
where ${H}_{l} \in \mathbb{R}^{n_{l}\times n_{l-1}}$,   $x_{l}\in \mathbb{R}^{n_l\times 1}$,  $l=1,...,l_f$, and $H_{l}^\top,H_{l}^+,H_{l}^\mp$ denote the transpose, pseudoinverse, and transposed pseudoinverse  of $H_{l}$, respectively.  $H_{l} $ and $H_{l}^+$ can be constructed as follows \cite{loukas2019graph}:
{\small\begin{equation}
    H_{l}(i,p)=\begin{cases}
\frac{1}{{|S^{(i)}_{l-1}|}} & \text{if p}  \in  S_{l-1}^{(i)} \\
0 & \text{otherwise } .
\end{cases}
\text{, }
H^+_{l}(p,i)=\begin{cases}
1 & \text{if p}  \in  S_{l-1}^{(i)} \\
0 & \text{otherwise. } 
\end{cases}
\end{equation}}

}}
\textbf{Spectral coarsening via local embedding.} The key to spectral coarsening in SF-SGL  is to determine the node aggregation sets. To this end, a low-pass graph filtering scheme has been adopted for   local spectral embedding, which can be achieved in linear time  \cite{zhao:dac19}. Let $b_{l}^{(k)}\in \mathbb{R}^{n_l\times 1}$  with $k=1,\cdots,K$ be the initial random vectors which are orthogonal to the all-one vector. After applying a few steps of Gaussian-Seidel relaxations for solving the linear system of equations $L_{G_{l}} b_{l}^{(k)}=0$ with $k = 1, \cdots, K$
 \cite{livne2012lean}, we can associate each node in the graph $G_l$ with a $K$-dimensional vector    using the embedding matrix  $B_l=[b_{l}^{(1)},...,b_{l}^{(K)}]$. Since each $b_{l}^{(k)}$ can be considered as the linear combinations of the first few Laplacian eigenvectors after smoothing (low-pass filtering), the embedding matrix $B_l$   can well approximate the eigensubspace matrix $U_r$ in (\ref{subspace}). Consequently, the nodes that are close to each other in the spectral embedding space  can be clustered into the same aggregation set. After finding all node  aggregation sets,   the node mapping matrix  $H_{l}$ that uniquely determines a coarse-level graph can be formed.
 Subsequently, the node feature matrix at level $l$ can also be obtained by $X_l=H_lX_{l-1}$, as shown in Figure \ref{fig:coarsening}. To better encode the feature information on the coarsened graph, we  propose the following scheme for updating edge weight between neighbouring nodes $s$ and $t$ based on their feature distance:
\begin{equation} 
    \omega_l(s,t)=1/\|X_l^\top e_{s,t}\|^2_2.
\end{equation}

\textbf{Bounding  approximation errors.} To ensure  the solution quality of the proposed graph learning framework, it is important to estimate and bound  the errors due to  spectral graph coarsening.  Graph ${{G}_{l}}$ is considered as a good approximation  of its finer version  ${G}_{l-1}$ if the following holds \cite{loukas2019graph}: 
\begin{equation}
\begin{split}
     \gamma_1 \lambda_{l-1}^{(i)}\leq
  \lambda_{l}^{(i)}\leq\gamma_2 {\small{\frac{(1+\epsilon)^2}{1-\tau\epsilon^2}}} \lambda_{l-1}^{(i)}    
 \\
     \frac{1}{\sigma_{l-1}}\|u^{(i)}_{l-1}\|_{L_{G_{l-1}}}\leq \|u^{(i)}_{l}\|_{L_{G_{l}}}\leq \sigma_{l-1}\|u^{(i)}_{l-1}\|_{L_{G_{l-1}}},
     \end{split}
\end{equation}
where  $l=1,...,l_f$, $\lambda_{l}^{(1)},..,\lambda_{l}^{(K)}$ denote the $K$ non-decreasing eigenvalues of $L_{G_{l}}$,    $  \tau=\lambda^{(K)}_{l-1}/\lambda^{(2)}_{l-1}$,  $\epsilon=(\sigma_{l-1}^2-1)/(\sigma_{l-1}^2+1)$ with $\sigma_{l-1}\leq (\frac{1+\sqrt{\tau}}{1-\sqrt{\tau}})^{\frac{1}{2}}$, and $\gamma_1 (\gamma_2$) denotes the smallest (largest) eigenvalue of $(H_lH_l^\top)^{-1}$.

\subsection{ Graph Topology Learning via A Bottom-up Approach}
\textbf{A spectrum decomposition perspective.} As shown in \cite{zhang2020sf}, the spectrally-coarsened graphs will carry different  (spectrum) bandwidths of   the original graph $G$. For example, the coarsest graph Laplacian $L_{G_{l_f}}$ will   only preserve   the   key spectral  properties of ${G_{0}}$, such as the first few  Laplacian eigenvalues and eigenvectors, whereas the Laplacians of the increasingly finer graphs   will  match the lower to moderate eigenvalues of $L_{G_{0}}$.   Consequently, {the coarsened graphs can be considered as a cascade of low-pass graph filters} with gradually decreasing bandwidths: the finest graph retains the highest bandwidth, whereas the coarsest graph   retains the lowest bandwidth.

\textbf{Mapping the learned graphs to finer levels.}  SF-SGL strives to {progressively form a series of increasingly finer  graphs such that the spectral embedding distortions can be significantly mitigated}.  As shown in Figure \ref{fig:backward}, to map the learned graph $P_l$ to a   finer level, two different types of edges  will be formed for constructing the  graph $P_{l-1}$ at the $(l-1)$-th level: (1) any node in $P_{l}$  corresponds to a node aggregation set in $P_{l-1}$, in which  an MST will be extracted to form  the \textbf{inner-cluster edges}; (2) any edge in   $P_{l}$  corresponds to at least one edge connecting between the two node aggregation sets in $P_{l-1}$, while only the edge with the largest weight will be kept as the \textbf{inter-cluster edge} for constructing $P_{l-1}$.

\textbf{Spectrally-critical edges at coarse levels.} As shown in Figure \ref{fig:emebedding}, in order to  identify and include spectrally-critical off-subgraph edges with large spectral embedding distortions, the aforementioned local spectral embedding will be applied for graph $P_{l-1}$ to generate its embedding matrix $B'_{l-1}=[b_{l-1}^{(1)},\cdots,b_{l-1}^{(K)}]$.
The spectral embedding distortion of the coarse level graph $P_{l-1}$ can be computed as

\begin{equation}
    \eta_{s,t}={z_{s,t}^{emb}}/{z_{s,t}^{data}},
\end{equation}
where $z_{s,t}^{emb}=\|B_{l-1}^{'\top} e_{s,t}\|^2_2$ and $z^{data}_{s,t} =\| X_{l-1}^\top e_{s,t}\|^2_2$.
A larger embedding distortion of an off-subgraph edge indicates that including this edge into the learned graph will more precisely encode the distances between the data points.
Note that the edges identified by SF-SGL on the coarsest graph will tend to connect the most distant nodes for drastically mitigating long-range spectral embedding distortions, thereby impacting  only the first few Laplacian  eigenvalues (low-frequency band); on the other hand, the edges identified  on the finest  graph   will tend to only connect   local nodes for mitigating short-range embedding distortions, thereby impacting relatively large    eigenvalues (high-frequency band).

\subsection{Algorithm Flow and Complexity of SF-SGL }\label{main:complexity}
The proposed multilevel graph learning framework (SF-SGL)   allows a flexible decomposition of the entire graph spectrum (to be learned)  into multiple different eigenvalue clusters, such that the most spectrally-critical edges that are the key to mitigating  various ranges of spectral embedding distortions can be more efficiently identified.
Algorithm \ref{alg:directed_graph_spar} shows the algorithm flow for the proposed SF-SGL framework. All the aforementioned steps in SF-SGL can be accomplished in nearly-linear time by leveraging recent high performance algorithms for kNN graph construction \cite{malkov2018efficient}, solver-free spectral graph coarsening and embedding, as well as fast Laplacian solver. Consequently, SF-SGL can be accomplished in nearly-linear time.

\begin{algorithm}[!htbp]
\small { \caption{The SF-SGL algorithm flow.} \label{alg:directed_graph_spar}
\begin{flushleft}
\hspace*{\algorithmicindent} \textbf{Input:} Voltage and current measurement matrices: $X \textit{, } Y\in \mathbb{R}^{N\times M}$, $k$ for initial kNN graph construction;\\
\hspace*{\algorithmicindent} \textbf{Output:} Learned graph $P$;
\end{flushleft}
    \begin{algorithmic}[1]

    \STATE{$P_{l} = \emptyset$ for $l = 0, ..., l_f$};\\
    \STATE{Construct an initial kNN graph $G$ as $G_0$ based on X};\\
     \STATE{ Perform spectral graph coarsening to get a series of coarsened kNN graphs  ${G_1},...,{G_{l_f}}$ as well as  the coarsening operators $H_1,...,H_{l_f}$ };\\
     \STATE{Let $X_0=X$ and calculate the feature matrices ${X_1},...,{X_{l_f}}$ by $X_{l}=H X_{l-1} $ for $l = 1, ..., l_f$; }\\
     \STATE{ Set $l = l_f$ and learn a sparse graph $P_{l}$  through SGL iterations; }\\
    \WHILE{ $ l\ge1$}
        \FOR{each node $i \in V_l$}
         \STATE{Find the induced subgraph $G_{l-1}[S_{l-1}^{(i)}]$ w.r.t. the aggregation set $S_{l-1}^{(i)}$; } \\
        \STATE{Extract the MST of the induced subgraph $G_{l-1}[S_{l-1}^{(i)}]$ and add its edges into  $P_{l-1}$ };
    
        \ENDFOR 
        \FOR{each edge $(s,t) \in \tilde{E}_l$}
            \STATE{Add the edge with maximum weight between sets $S_{l-1}^{(s)}$ and $S_{l-1}^{(t)}$  into graph $P_{l-1}$} ;
       
        \ENDFOR 
        
        \STATE{Compute   $\eta_{s,t}$ for each candidate edge $(s,t) \in E_{l-1}\setminus \tilde{E}_{l-1}$};
        
        \STATE{Add the candidate edges with large $\eta_{s,t}$ values into $P_{l-1}$};
    \STATE{$l=l-1$}
 
    \ENDWHILE\\
     \STATE{Scale up the edge weights in $P_0$ and return the learned graph ${P}$}.
    \end{algorithmic}   }
\end{algorithm}
\section{Experimental Results}\label{result_sec}
The proposed SGL and SF-SGL algorithms have been implemented in Matlab and $C$++. The test cases  in this paper have been selected from a great variety of matrices that have been used in circuit simulation and finite element analysis problems. Note that the prior state-of-the-art graph learning algorithms \cite{egilmez2017graph} can take excessively long time even for very small graphs.  For example, learning the   test case ``airfoil" graph with $4,253$ nodes   using the CGL algorithm \cite{egilmez2017graph} can not be completed after $12$ hours,   whereas  the proposed SGL algorithm   only takes about $4$ seconds to achieve a high-quality solution;   learning the ``yaleShieldBig" graph   with $1,059$ nodes using   CGL       takes $1,235$ seconds, whereas   SGL   only takes about $2$ seconds (over $600\times$ speedups) for achieving similar  solution (spectrum preservation) quality. Therefore, in this work we will only compare with the graph construction method based on the standard kNN algorithm. All of our experiments have been conducted using a single CPU core of a computing platform running 64-bit RHEW 7.2 with a $2.67$GHz 12-core CPU and $50$ GB memory. 

\vspace{-0pt}

\subsection{Experimental Setup}
To create the linear voltage and current measurements, the following procedure has  been applied: (1) we first randomly generate $M$ current source vectors with each element sampled from a standard normal distribution; (2) each current vector will be normalized and orthogonalized to the all-one vector; (3) $M$ voltage vector measurements will be collected by solving the original graph Laplacian matrix  with the $M$ current vectors as the (right-hand-side) input vectors; (4)  the voltage and current vectors will be stored in data matrices $X=[x_1,...,x_M]$ and $Y=[y_1,...,y_M] \in {\mathbb{R} ^{N\times M}}$, respectively, which will be used as the input data of the proposed SGL and SF-SGL algorithms. By default, $M=50$ is used for generating the voltage and current measurements. By default, we choose $k=5$ for constructing the kNN graph for all test cases in SGL, and set $r=5$ for constructing the spectral embedding  matrix in (\ref{subspace}). The edge sampling ratio  $\beta=10^{-3}$  has been used. The SGL iterations will be terminated if   $s_{max}<tol=10^{-12}$. When approximately computing the objective function value (\ref{opt2}), the first $50$ nonzero Laplacian eigenvalues are used. 

To  visualize the structure of each graph, the spectral graph drawing technique has been adopted \cite{koren2003spectral}: to generate the 2D   graph   layouts, each entry of the first two nontrivial Laplacian eigenvectors ($u_2, u_3$) will be used as  the $x$ or $y$   coordinate for embedding each node (data point), respectively. We also assign  the nodes with the same   color if they belong to the same node cluster determined by spectral graph clustering \cite{zhao2021wsdm}. 
\vspace{-0pt}

\subsection{Comprehensive Results for Graph Learning with SGL} 
 
   \textbf{Algorithm convergence.} As shown in Figure \ref{fig:distRes}, for the ``2D mesh" graph ($|V|=10,000$, $|E|=20,000$) learning task, SGL  requires about $40$ iterations to converge to $s_{max}\le 10^{-12}$ when starting from an initial  MST  of a 5NN graph.  
  \begin{figure}
     \centering\includegraphics[width=0.795\linewidth]{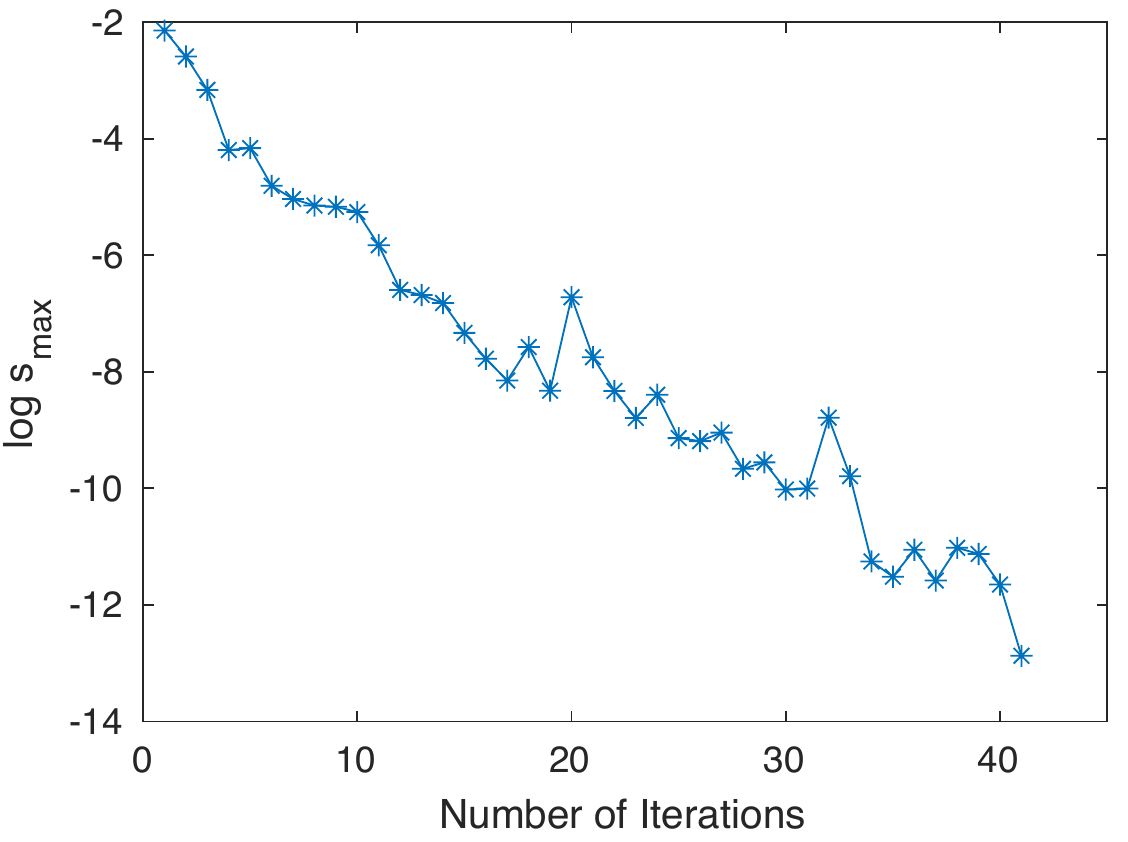}
  \caption{The decreasing maximum sensitivities (``2D mesh" graph). }\label{fig:distRes}
  \end{figure}
  
    \begin{figure}
   \centering\includegraphics[width=0.7985\linewidth]{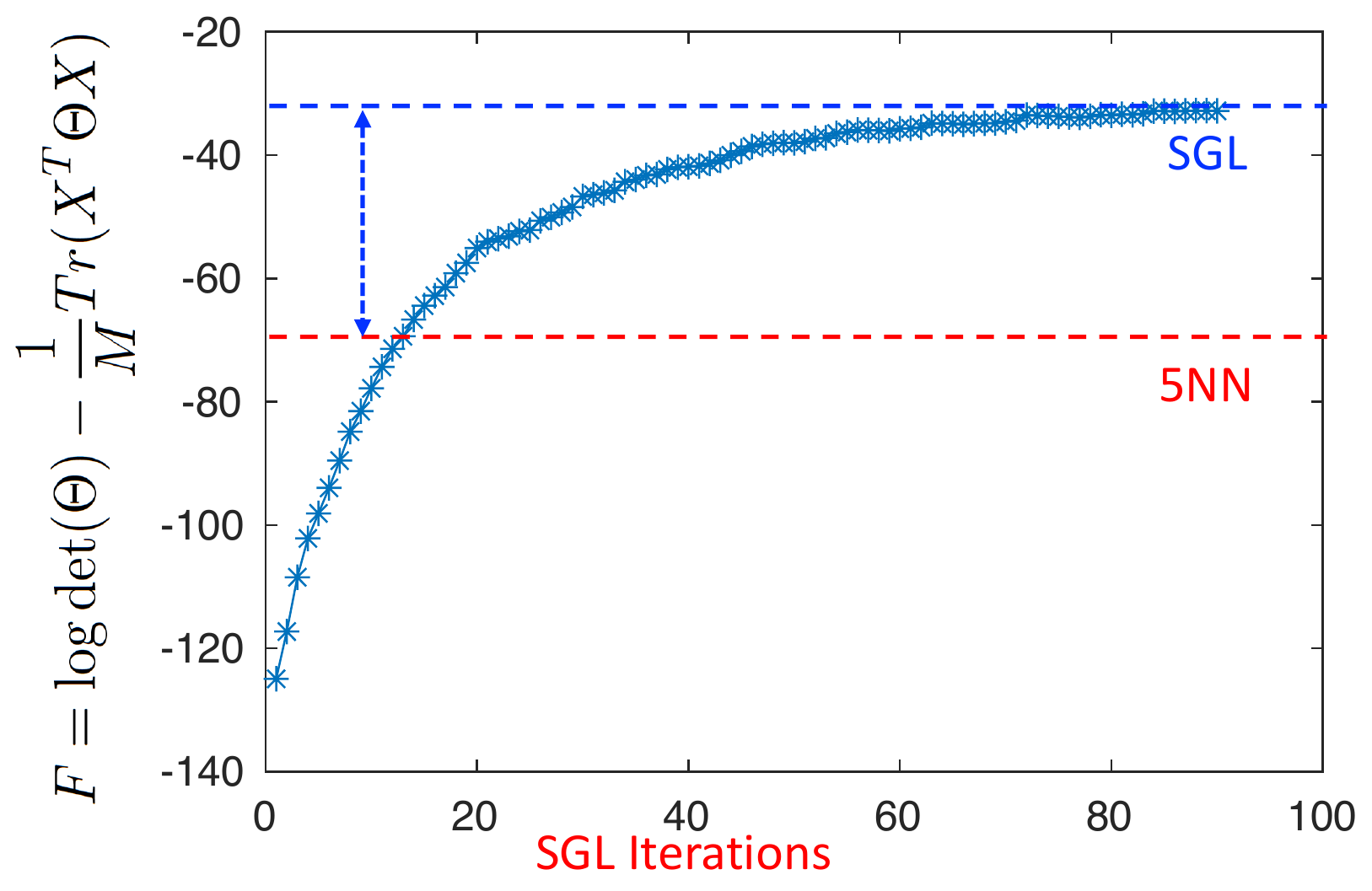}
  \caption{The objective function values (``fe\_4elt2" graph). }\label{fig:compKNN}
  \end{figure}

   \textbf{Comparison with kNN graph.} As shown in Figure \ref{fig:compKNN}, for the ``fe\_4elt2" graph ($|V|=11,143$, $|E|=32,818$) learning task, SGL  converges in about  $90$ iterations when starting from an initial  MST of a 5NN graph. For the 5NN graph, we do the same spectral edge   scaling. As shown in Figures \ref{fig:compKNN} and  \ref{fig:compKNN2}, the SGL-learned graph achieves a more optimal objective function value  and a much better  spectral approximation   than the 5NN graph. As observed, the SGL-learned graph has a density similar to a spanning tree, which is much sparser than the 5NN graph.
   \begin{figure}[!htb]
  \centering\includegraphics[width=0.995\linewidth]{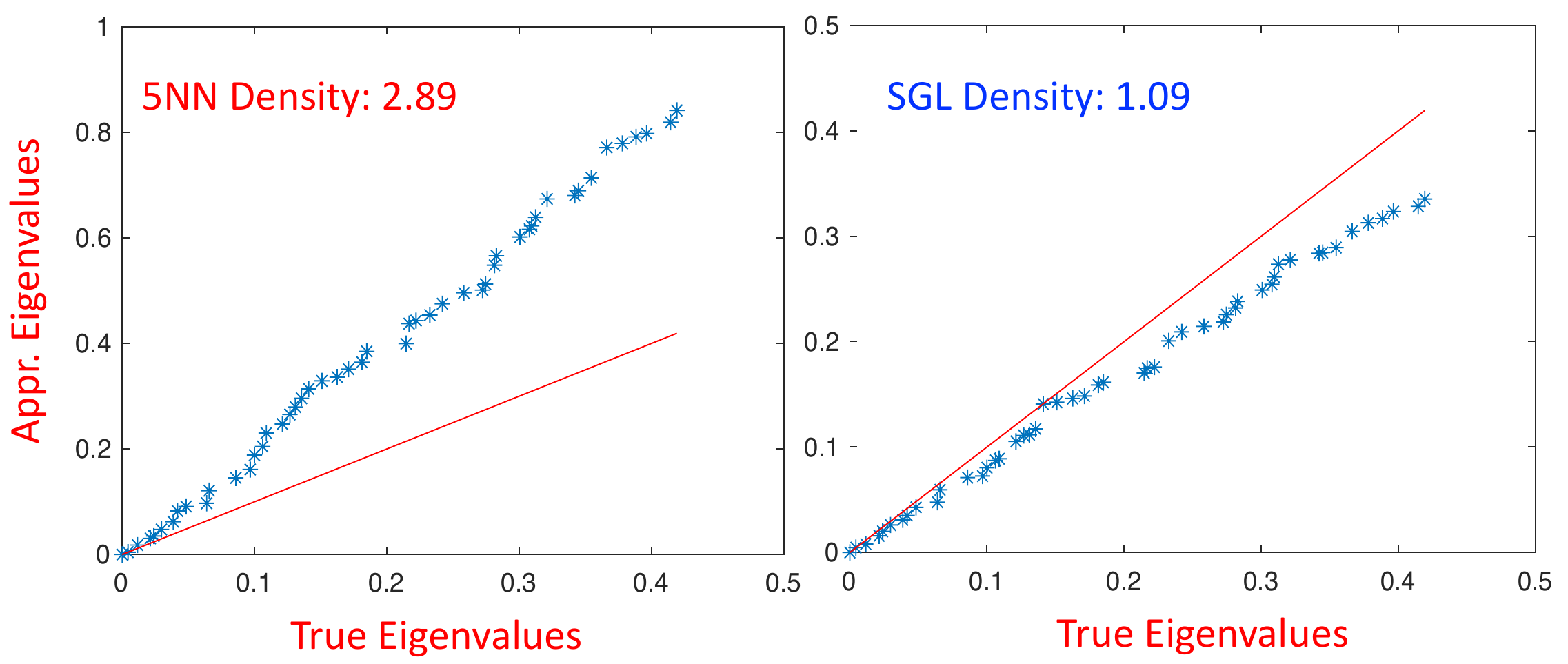}
  \caption{The comparison with a 5NN graph (``fe\_4elt2" graph)}\label{fig:compKNN2}
  \end{figure}
  
\textbf{Learning circuit networks.} As shown in Figures \ref{fig:sgl_airfoil}, \ref{fig:sgl_crack} and \ref{fig:G2}  for the ``airfoil" ($|V|=4,253$, $|E|=12,289$), the ``crack" ($|V|=10,240$, $|E|=30,380$),  
and  the ``G2\_circuit" ($|V|=150,102$, $|E|=288,286$) graphs,   SGL   can consistently learn     sparse graphs which are slightly denser than  spanning trees while preserving the key graph spectral  properties. In Figure \ref{fig:resistance}, we observe highly correlated results when comparing the  effective resistances computed on the original  graphs with the  graphs learned by SGL.

\textbf{Learning reduced networks.} As shown in Figure \ref{fig:G2Redu}, by randomly choosing a small portion of node voltage measurements (without using any currents measurements in $Y$ matrix) for graph learning, SGL can learn  spectrally-similar graphs of much smaller sizes: when $20\%$ and  $10\%$  node voltage measurements are used for graph learning, $5\times$ and $10\times$ smaller resistor networks can be constructed, respectively, while preserving the key spectral (structural) properties of the original graph.

  \begin{figure}
  \includegraphics[width=0.9995\linewidth]{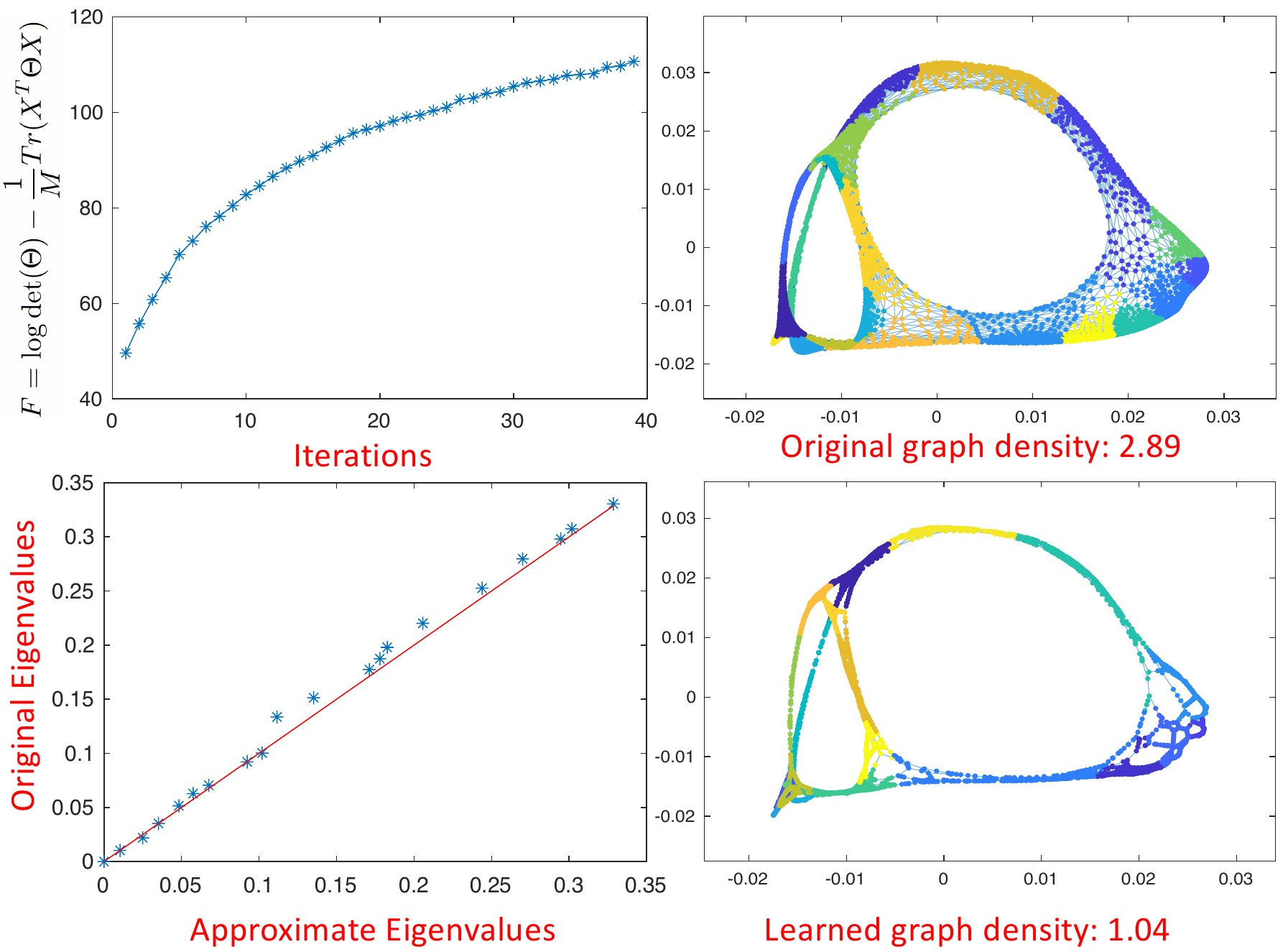}
  \caption{The results for learning the ``airfoil" graph.}\label{fig:sgl_airfoil}
  \end{figure}
  
    \begin{figure}
  \includegraphics[width=0.995\linewidth]{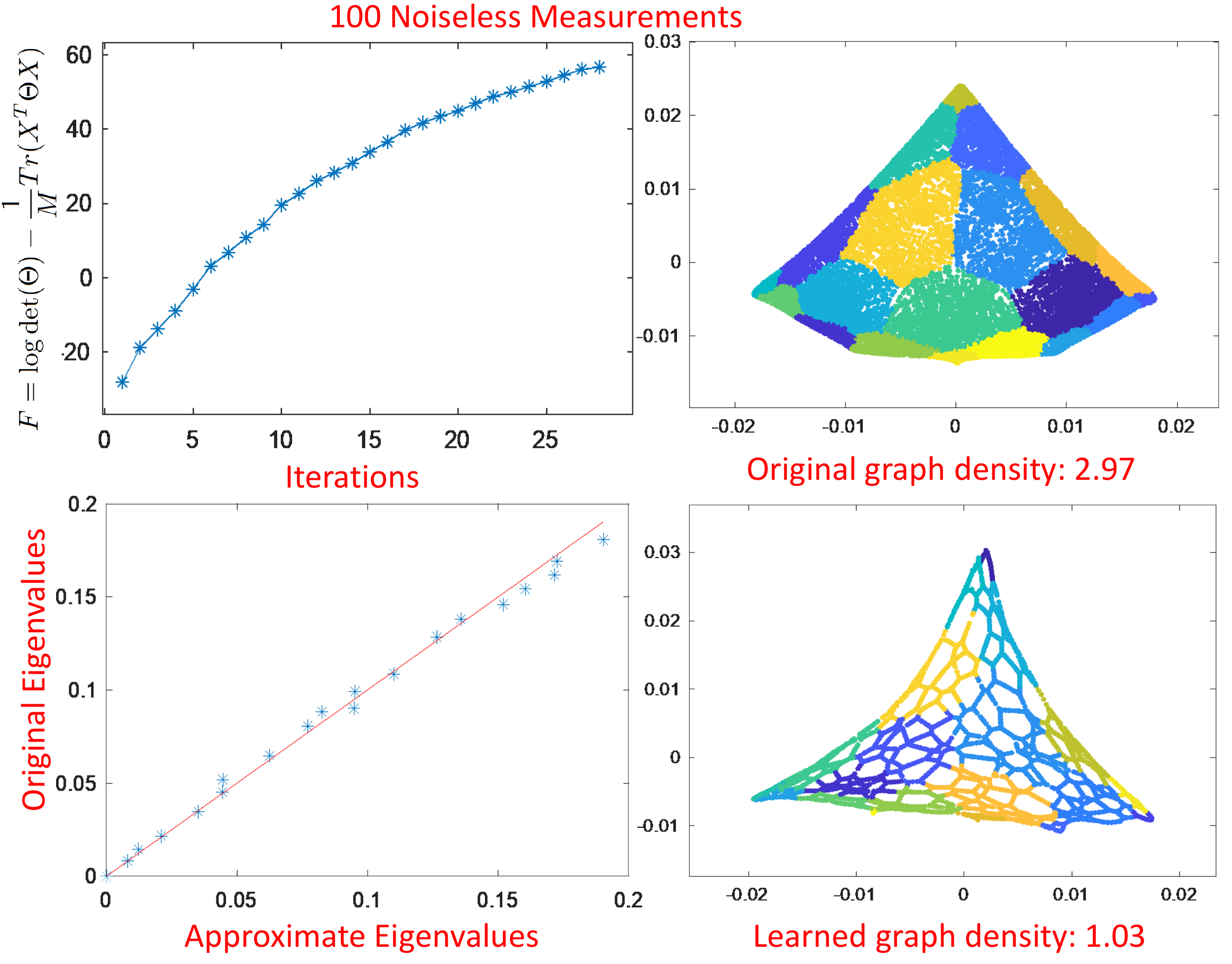}
  \caption{The results for learning the ``crack" graph.}\label{fig:sgl_crack}
  \end{figure}
  
   \begin{figure}
  \includegraphics[width=0.9995\linewidth]{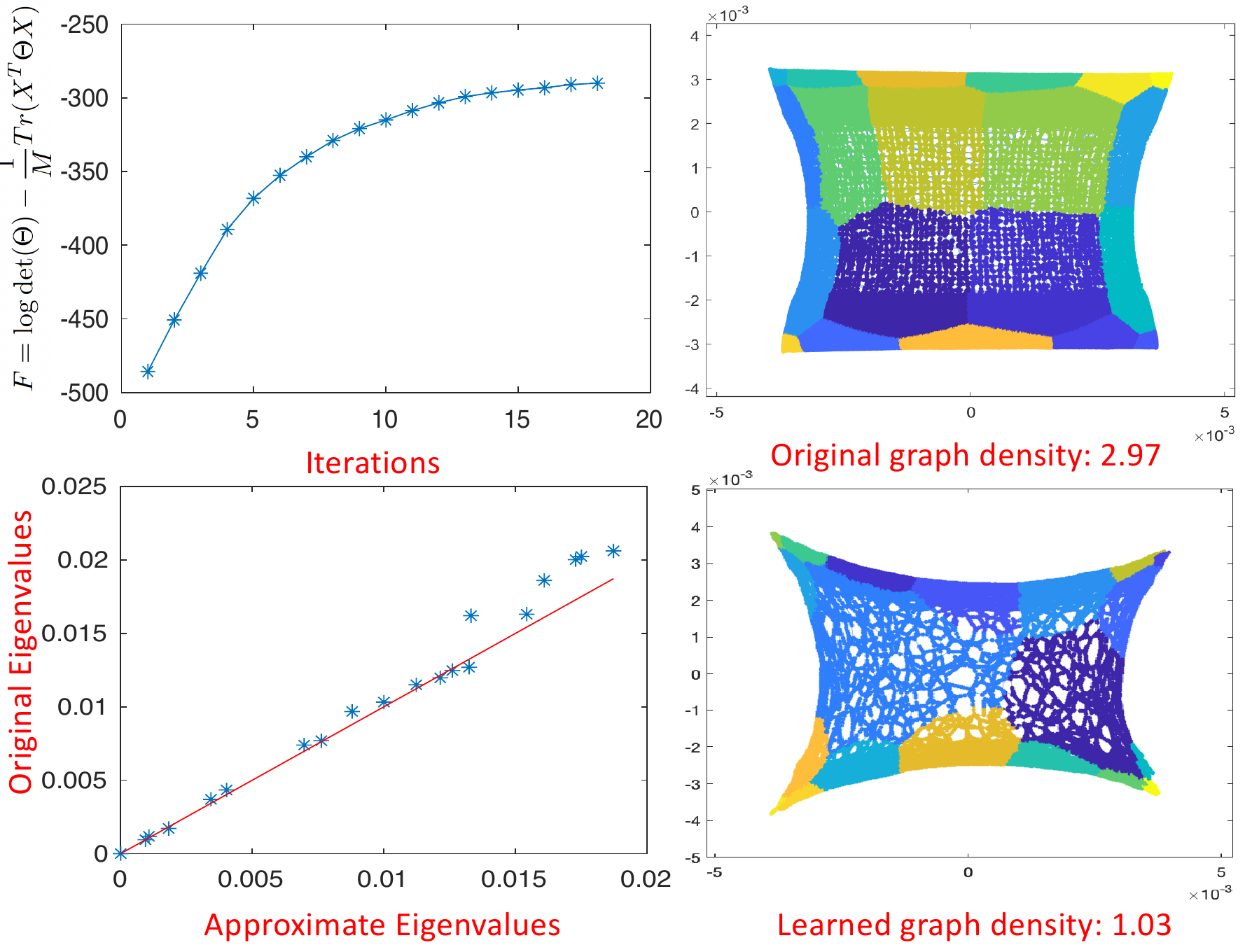}
  \caption{The results for learning the ``G2\_circuit" graph.}\label{fig:G2}
  \end{figure}

    \begin{figure}
  \includegraphics[width=0.9995\linewidth]{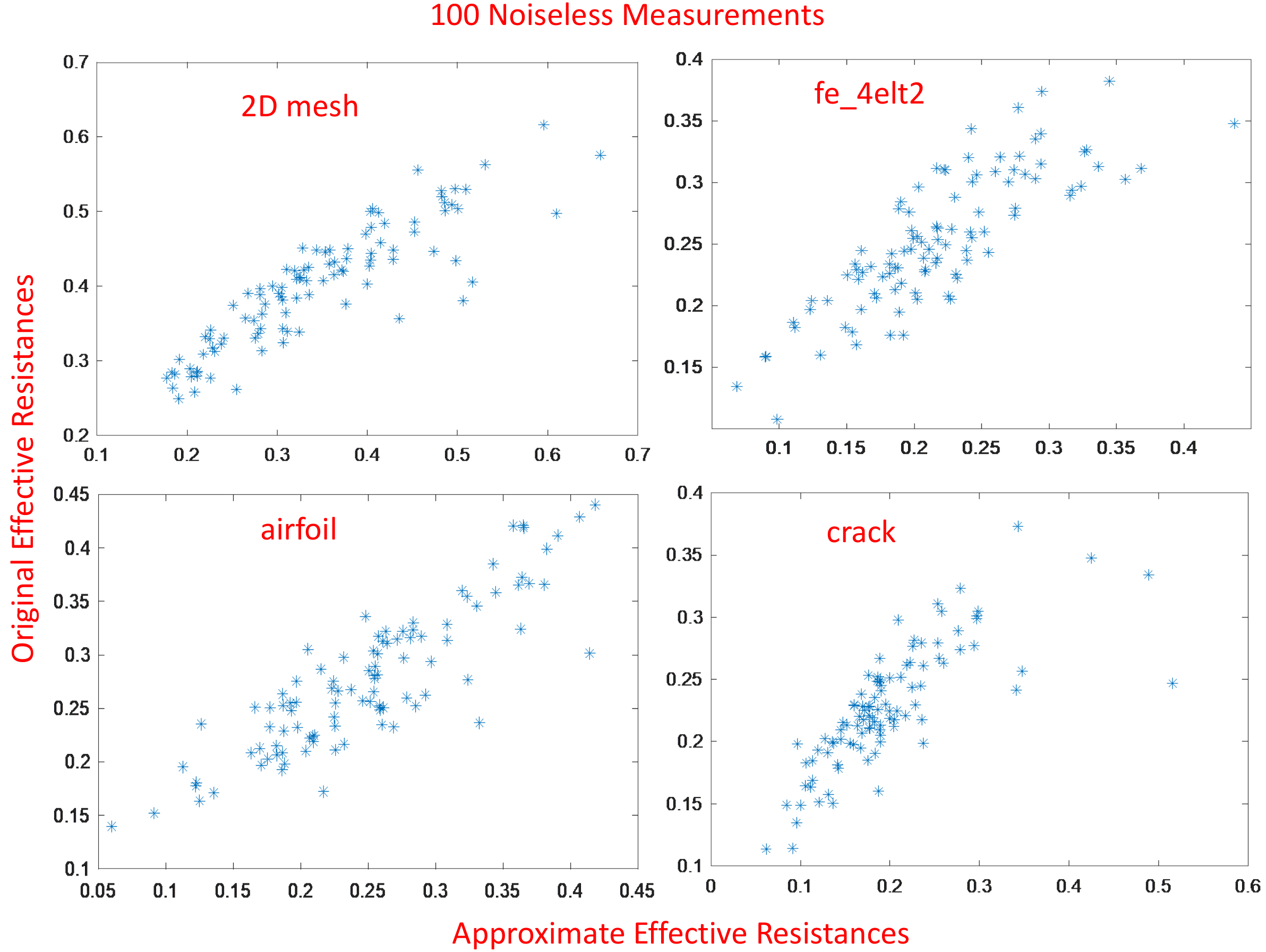}
  \caption{The effective resistances correlations (scatter plots). }\label{fig:resistance}
  \end{figure}
  
  \begin{figure}
  \includegraphics[width=0.999\linewidth]{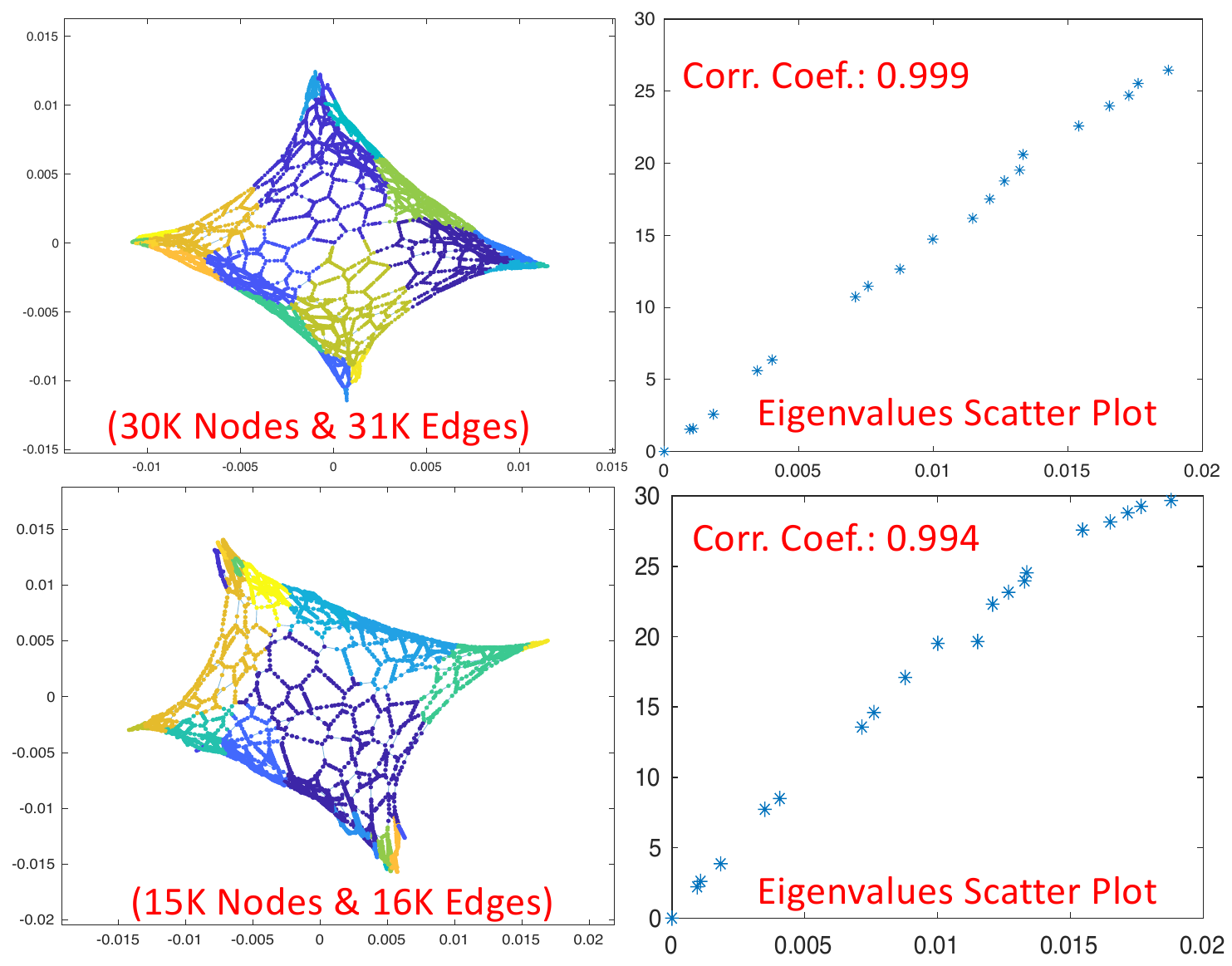}
  \caption{The reduced graphs learned by SGL (``G2\_circuit" ).}\label{fig:G2Redu}
  \end{figure}
  
\textbf{Learning with noisy measurements.} We show the results of the ``2D mesh" graph learning  with noisy voltage measurements. For each SGL graph learning task, each input voltage measurement (vector) $\tilde x$ will be computed by: $\tilde x= x+ \zeta\| x\|_2\epsilon$, where $\epsilon$ denotes a normalized Gaussian noise vector, and $\zeta$ denotes the noise level. As shown in Figure \ref{fig:MeshNoise}, the increasing  noise levels will result in worse approximations of the original   spectral properties. It is also observed that even with a very significant noise level of $\zeta=0.5$, the graph learned by the proposed SGL algorithm can still preserve the first few Laplacian eigenvalues that are key to the graph structural (global) properties.

   \begin{figure}
  \includegraphics[width=0.995\linewidth]{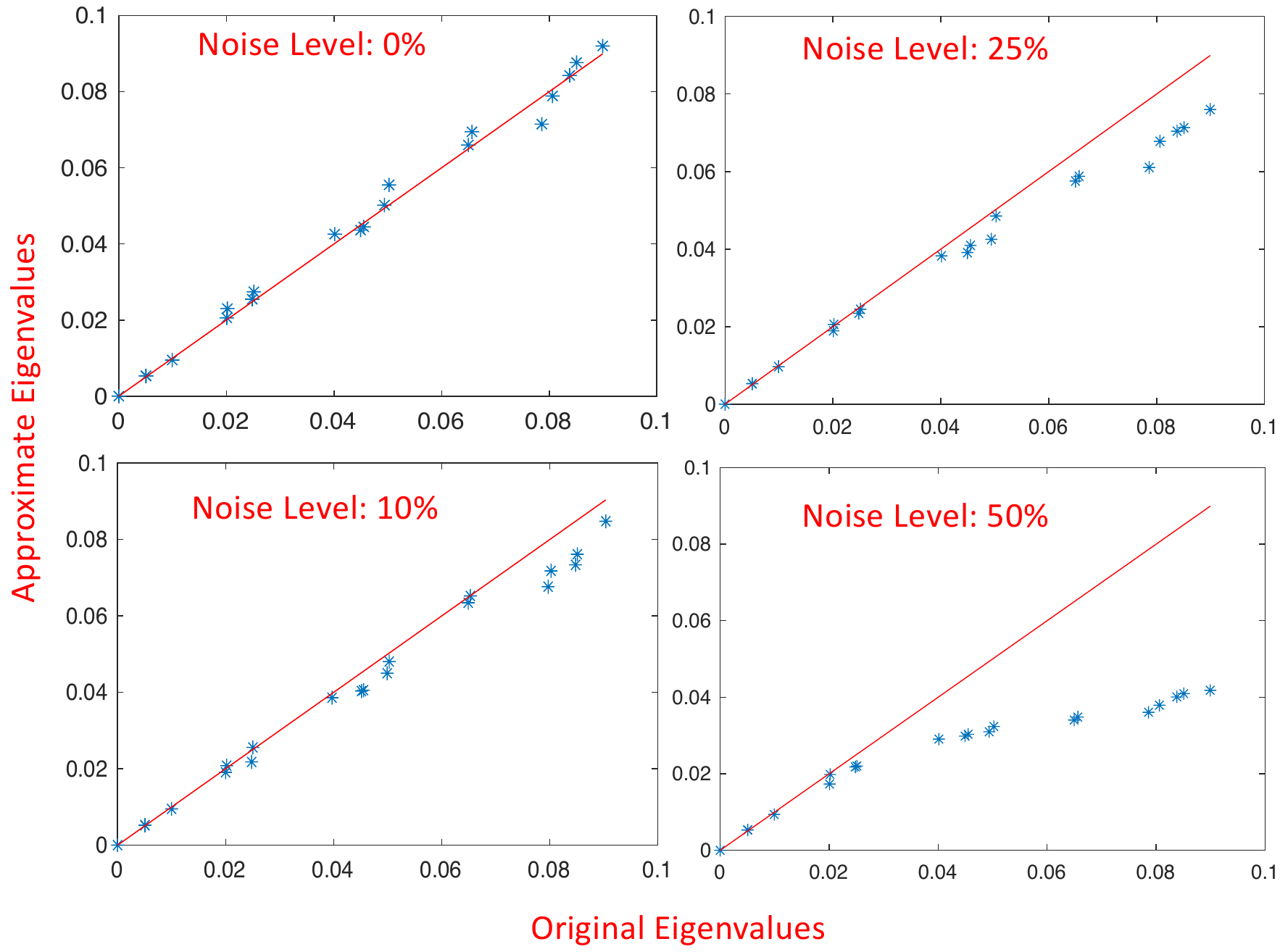}
  \caption{The graphs learned with noises (``2D mesh" graph).}\label{fig:MeshNoise}
  \end{figure}
  \begin{figure}
  \includegraphics[width=0.9995\linewidth]{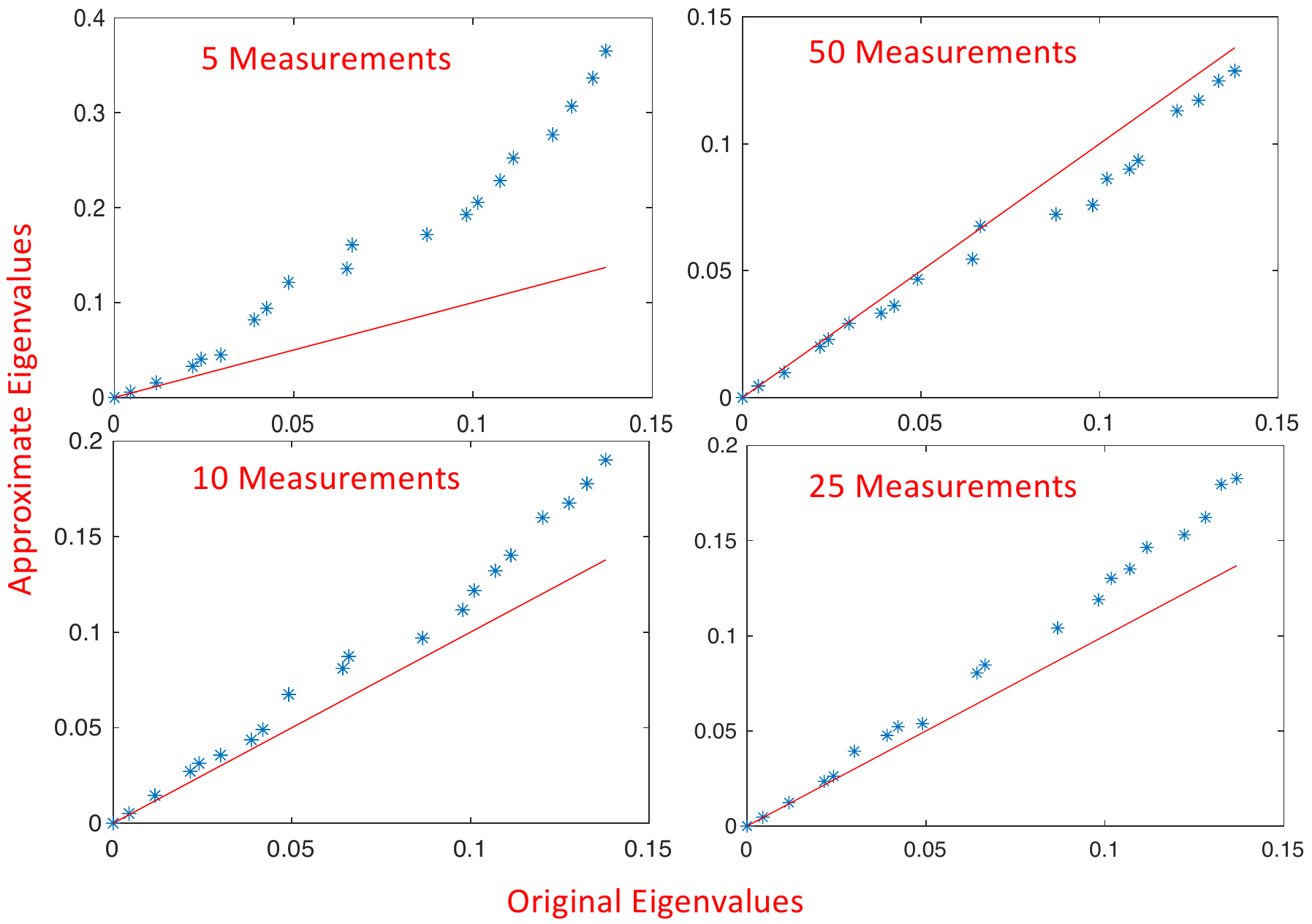}
  \caption{The effect of the number of measurements (``fe\_4elt2" graph). }\label{fig:sampleComplexity}
  \end{figure}
     \begin{figure}
   \centering\includegraphics[width=0.78259835\linewidth]{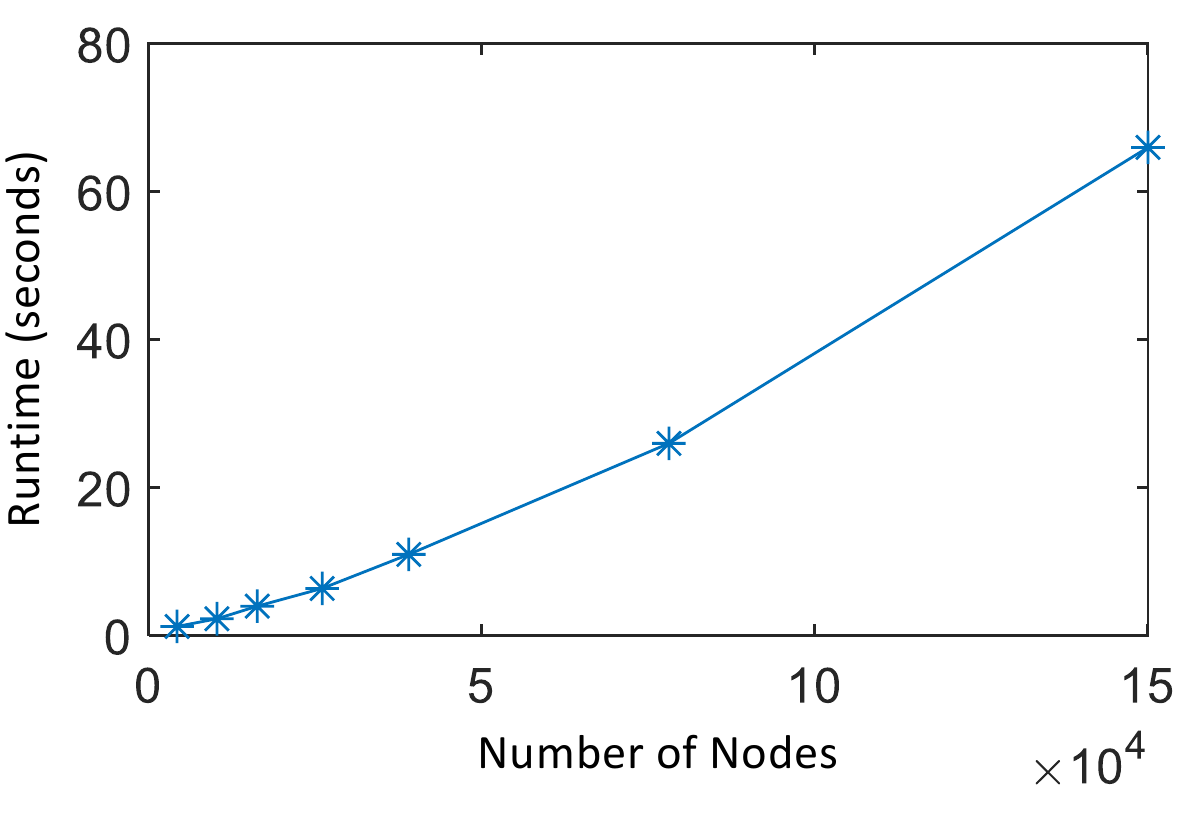}
    \vspace{-0pt}
  \caption{The runtime scalability of the SGL algorithm.}\label{fig:sgl_runtime}
  \vspace{-5pt}
  \end{figure}
\textbf{Sample complexity and runtime scalability} Figure \ref{fig:sampleComplexity} shows how the sample complexity (number of measurements)
may impact the graph learning quality. As observed, with increasing number of samples (measurements), substantially improved approximation of the  graph spectral    properties can be achieved. In the last, we show the runtime scalability of the proposed SGL algorithm in Figure \ref{fig:sgl_runtime}. The runtime includes the total time of Step 2 to Step 5 but does not include the time for Step 1. Note that  modern kNN algorithms can achieve highly scalable runtime performance \cite{malkov2018efficient}. 
    
       \begin{figure}
  \centering\includegraphics[width=0.7595\linewidth]{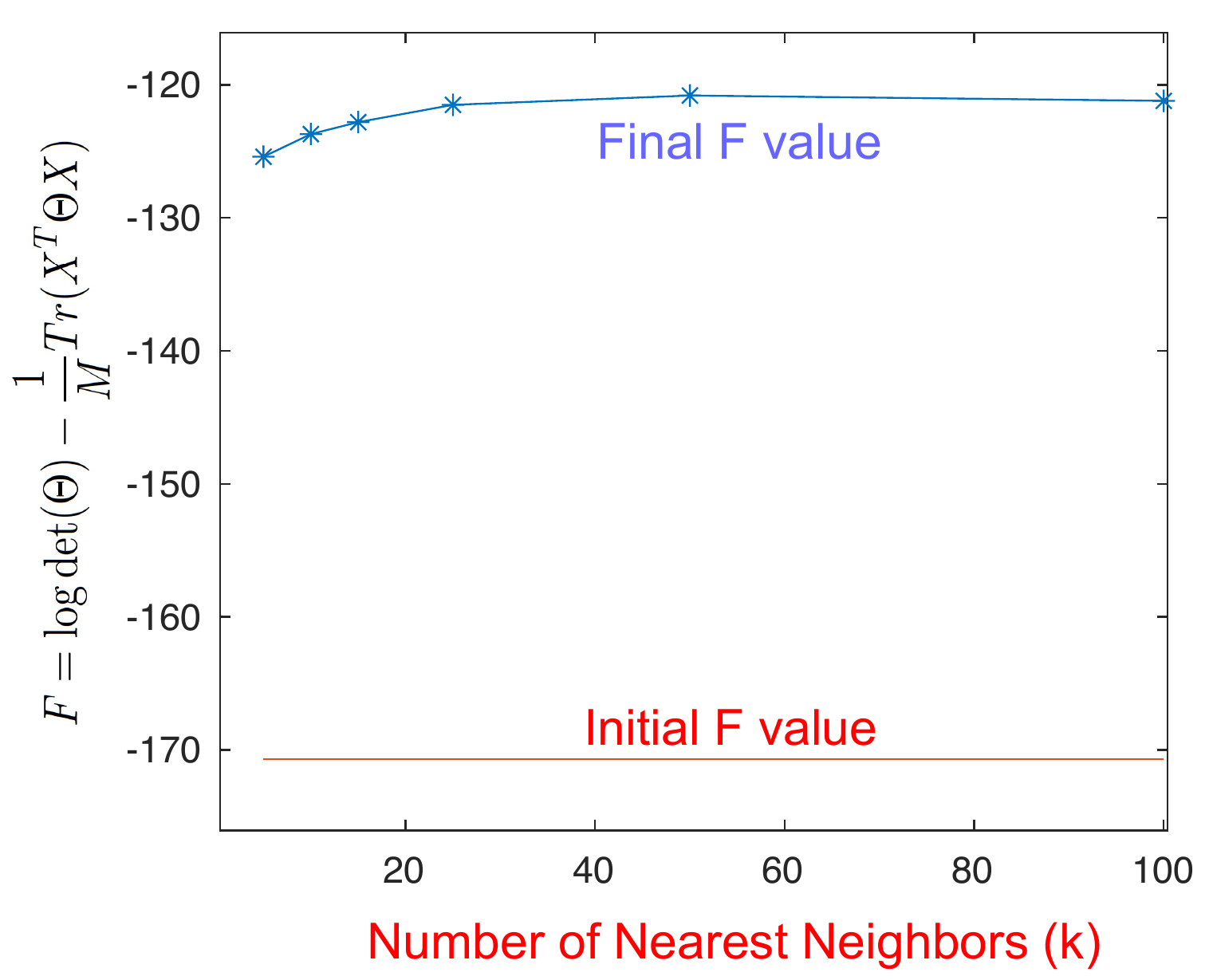}
  \caption{The final object function values achieved with various kNN graphs.}\label{fig:objknn}
  \end{figure}
  
        \begin{figure}
  \centering\includegraphics[width=0.8\linewidth]{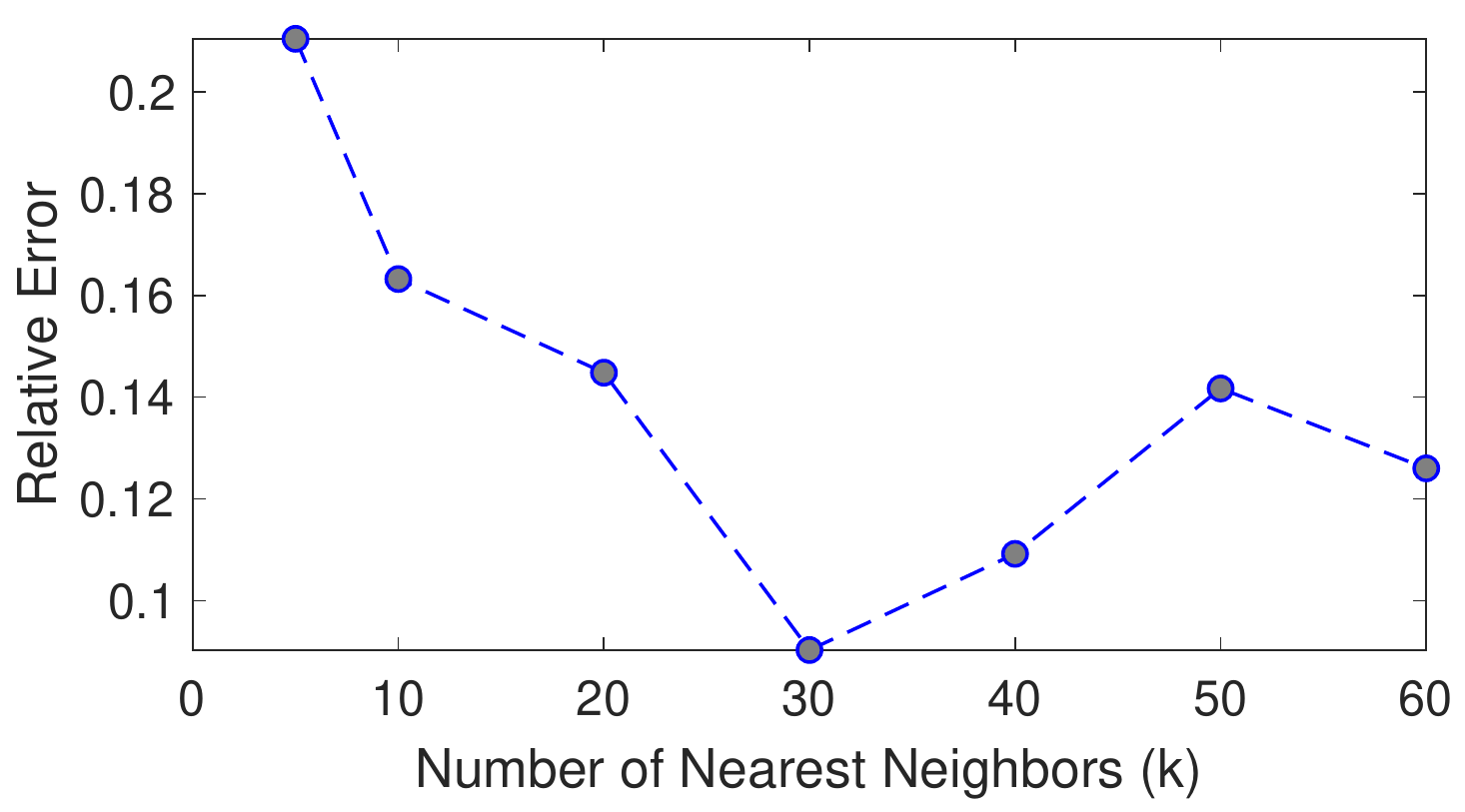}
  {\caption{The average relative errors of the first 50 eigenvalues between learned graph and original graph with different $k$ (the ``fe$\_$4elt" graph).}\label{fig:errknn}}
  \end{figure}
      
\textbf{Ablation analysis of SGL.} We provide comprehensive ablation analysis for the proposed SGL algorithm. In Figure \ref{fig:objknn}, we show that different choices of the number of nearest neighbors for constructing kNN graphs will only slightly impact the final solution quality (objective function values) of SGL. {{Figure \ref{fig:errknn} shows the average relative errors of the first $50$ eigenvalues between the learned graph and the original graph with various number of the nearest neighbors $k$. For example, given the first $50$ pairs of learned graph eigenvalues $\tilde{\lambda}_i$ and the original graph eigenvalues $\lambda_i$, the average relative error can be computed as $\frac{1}{50}\sum_{i=1}^{50}{\frac{|\tilde{\lambda}_i-\lambda_i|}{\lambda_i}}$.  It can be observed that the learned graphs can well preserve the spectral properties of the original graph with different $k$ for initial kNN graphs.  }} In Figure \ref{fig:embdim}, it is observed that with a higher dimension $r$ for spectral graph embedding, more accurate preservation of Laplacian eigenvalues in the learned networks can be achieved. In Figure \ref{fig:edgeiter}, we show that adding more edges in each SGL iteration will substantially improve the overall graph learning runtime without  compromising the solution quality.  
   
      \begin{figure}
  \centering\includegraphics[width=0.995\linewidth]{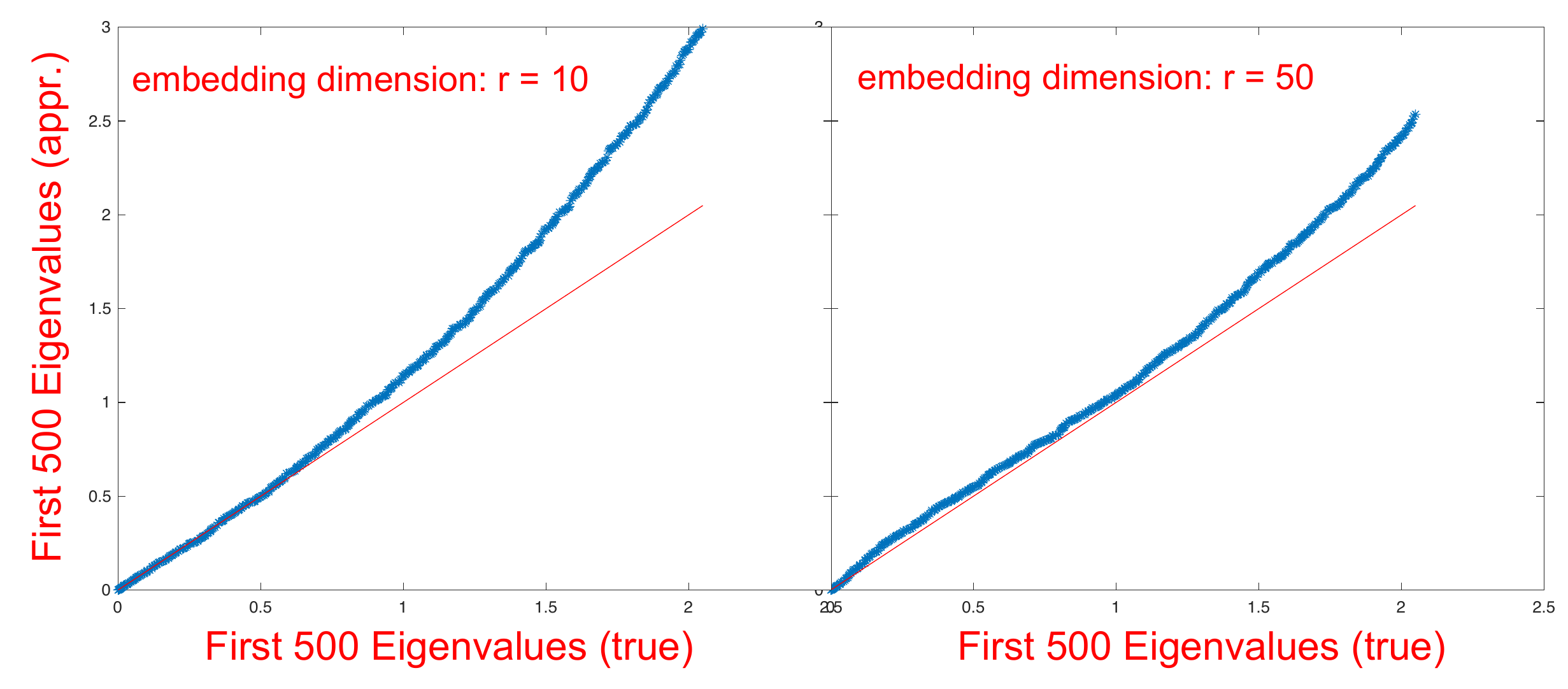}
  \caption{SGL results with different spectral embedding dimensions $r$.}\label{fig:embdim}
  \end{figure}
      \begin{figure}
  \centering\includegraphics[width=0.995\linewidth]{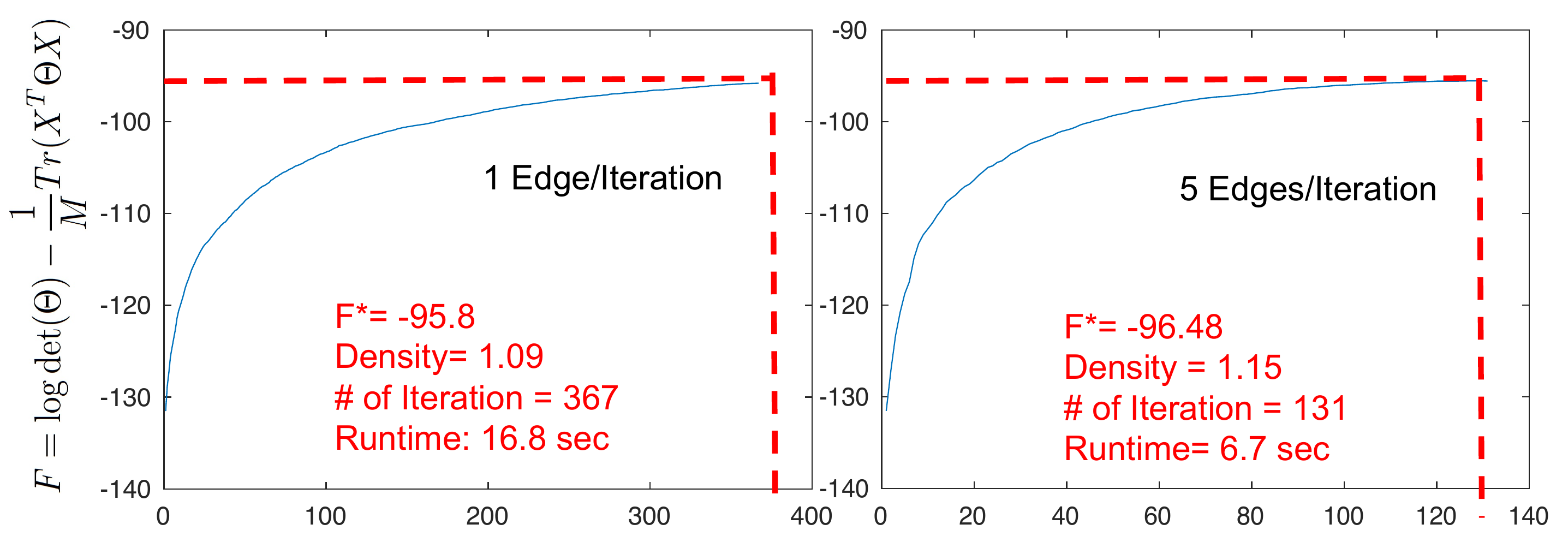}
  \caption{SGL iterations with different number of edges added per iteration.}\label{fig:edgeiter}
  \end{figure}
\subsection{Comprehensive Results for Graph Learning with SF-SGL}\label{subsection:SF-SGL}
\begin{figure}
\centering
  \includegraphics[width=\linewidth]{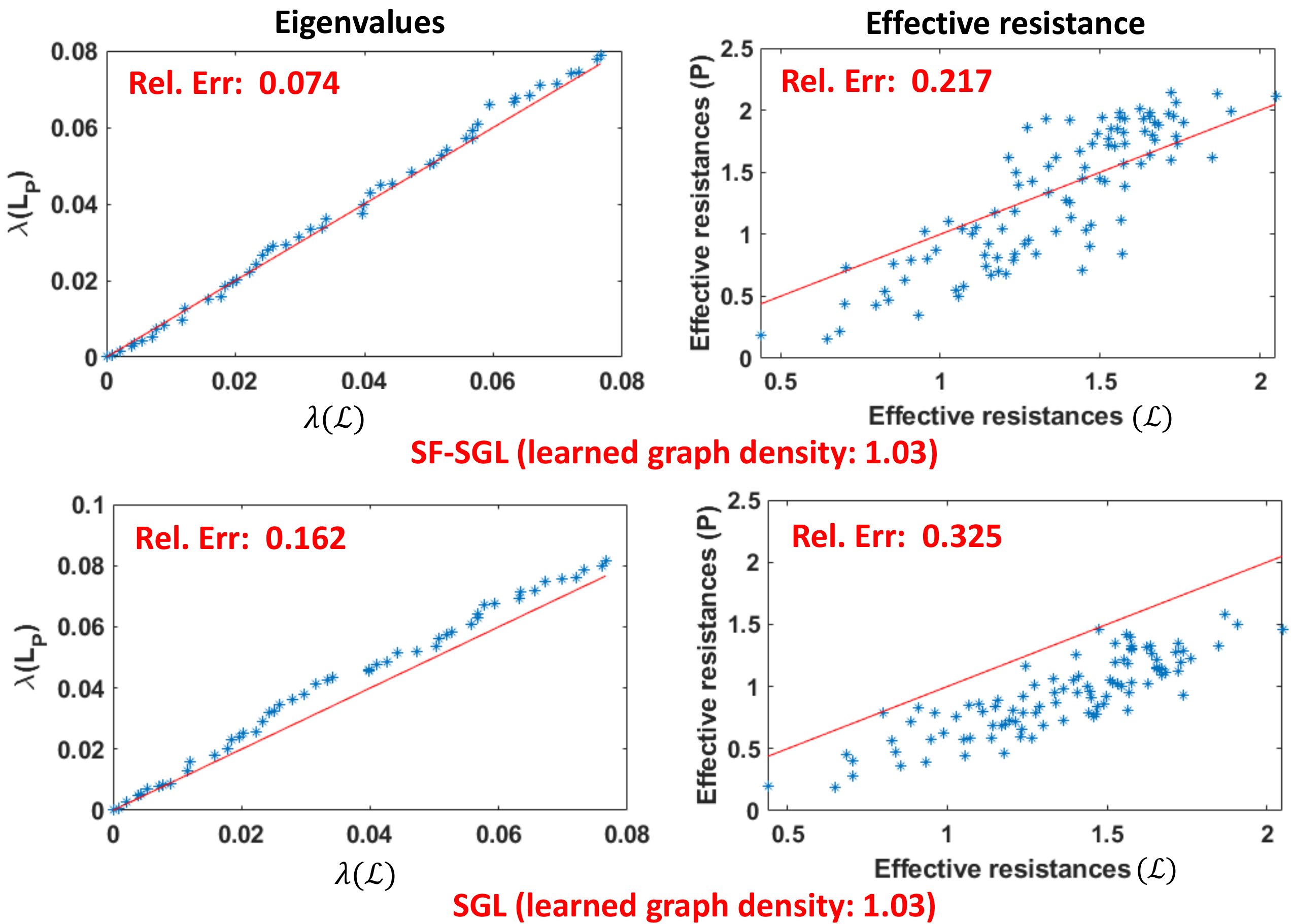}
  \caption{SF-SGL vs SGL on learned graphs (``fe$\_$4elt").}
  \label{fig:4elt}
\end{figure}

\textbf{Spectrum/resistance preservation.} Figure \ref{fig:4elt} shows the learned graph comparison using the proposed SF-SGL and SGL algorithms on graph ``fe$\_$4elt", where the first $50$ eigenvalues are compared between the learned graphs and the original graph. Meanwhile, we randomly sample $100$ pairs of the nodes to compute their effective resistances on the learned graphs and the original graph. $\textbf{Rel. Err}$ represents the average relative errors given $k$ pairs of the approximated solution $\tilde{a}$ and the true solution $a$, such that $\textbf{Rel. Err} = \frac{1}{k}\sum_{i=1}^k{\frac{|\tilde{a}-a|}{a}}$. Under this condition, $k$ will be $50$ for eigenvalue comparison and $100$ for effective resistance comparison. It can be observed that the graph learned with SF-SGL can better preserve the spectral properties of the original graph.

\begin{figure}
\centering
 \includegraphics[width=\linewidth]{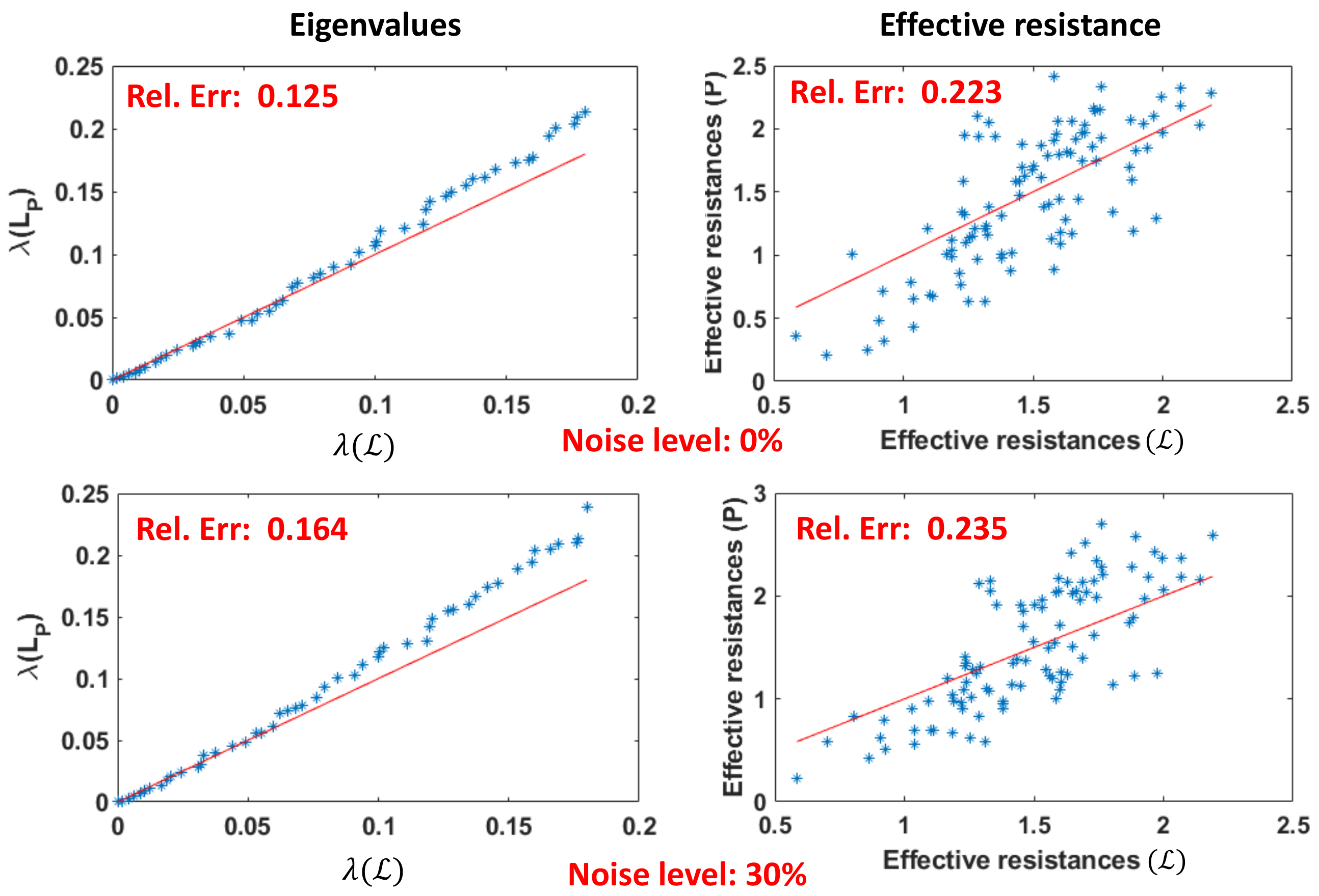}
  \caption{Learned graphs with noise (the ``airfoil" graph).}
  \label{fig:airfoil}
\end{figure}

\textbf{Learning with noises.} Figure \ref{fig:airfoil} shows the results of the ``airfoil" graph learning with noisy voltage measurements, where graph density of the learned graph is $1.04$.   We   observe that the measurement noises will result in worse approximations of the original spectral graph properties. However, the learned graph can still well preserve the key spectral properties of the original graph even with $30\%$ of noise level in the measurement data.

\begin{table*}
 \centering \caption{Comparison of spectral graph learning results between SF-SGL and SGL. }
 \resizebox{\textwidth}{!}{
 \begin{tabular}{|c|c|c|c|c|c|c|c|c|c|c|c|c|}
  \hline
 \multirow{2}{*}{Test cases} & {\multirow{2}{*}{$|V|$}} & {\multirow{2}{*}{$|E|$}} & 
 {\multirow{2}{*}{$\tau_G$}} &
 {\multirow{2}{*}{$\tau_P$}} &\multicolumn{4}{|c|}{SGL } & \multicolumn{4}{|c|}{SF-SGL }\\
\cline{6-13}
 { } & { } & { } & { } & { }& {$Err(\lambda)$} & {$Err(R)$} & {$T_P$} & {$T$} & {$Err(\lambda)$} & {$Err(R)$} & $T_P$ & {$T$}\\
  \hline
  {airfoil} & {$4.3E3$} & {$1.2E4$} & {$2.89$} & {$1.04$} & {$0.228$} & {$0.296$} & {$5.5s$} & {$7.5s$} & {$\mathbf{0.125} $} & {$\mathbf{0.223}$} & {$0.2s$ $(27X)$} & {$0.3s$ $(25X)$} \\
  \hline
  {fe$\_$4elt} & {$1.1E4$} & {$3.3E4$} & {$2.97$} & {$1.03$} & {$0.162$} & {$0.325$} & {$10.6s$} & {$16.8s$} & {$\mathbf{0.074}$} & {$\mathbf{0.217}$} & {$0.5s$ $(21X)$}& {$0.7s$ $(24X)$} \\
  \hline
  {crack} & {$1.0E4$} & {$3.0E4$} & {$2.97$} & {$1.04$} & {$\mathbf{0.049}$} & {$0.271$} & {$10.1s$} & {$15.8s$} & {$0.176$} & {$\mathbf{0.141}$} & {$0.7s$ $(14X)$} & {$0.9s$ ($17X$)} \\
    \hline
 {fe$\_$sphere} & {$1.6E4$} & {$4.9E4$} & {$3.00$} & {$1.05$} & {$0.084$} & {$0.243$} & {$19.5s$} & {$29.9s$} & {$\mathbf{0.078}$} & {$\mathbf{0.163}$} & {$1.3s$ $(15X)$}& {$1.8s$ $(16X)$} \\
  \hline
  {2Dmesh} & {$2.5E5$} & {$5.0E5$} & {$2.00$} & {$1.02$} & {$0.136$} & {$0.263$} & {$356.7s$ } & {$2,464.7s$} & {$\mathbf{0.062}$} & {$\mathbf{0.119}$} & {$21.3s$ ($17X$)} & {$31.6s$ ($78X$)} \\
  \hline
  {3Dmesh} & {$2.5E5$} & {$7.4E5$} & {$2.95$} & {$1.12$} & {$0.207$} & {$0.577$} & {$1,054.7s$ } & {$3,155.7s$} & {$\mathbf{0.125}$} & {$\mathbf{0.260}$} & {$36.8s$ ($28X$)} & {$47.8s$ ($66X$)} \\
   \hline
 {mat3} & {$4.0E5$} & {$1.2E6$} & {$2.89$} & {$1.13$} & {-} & {-} & {-} & {-} & {$0.236$} & {$0.350$} & {$67.7s$} & {$86.7s$} \\

  \hline\end{tabular}\label{table:sf_sgl}
  }
 \end{table*}
 
\textbf{More detailed  results.} Table \ref{table:sf_sgl} shows more comprehensive results on learned graphs using the proposed SF-SGL and SGL algorithms, where $\tau_G$ and $\tau_P$ represent the graph densities ($|E|/|V|$) for the original graph and the learned graph, respectively; $Err(\lambda)$ denotes the average relative error of the first $50$ eigenvalues between the original graph and the learned graph, and $Err(R)$ is the average relative error of the effective resistances measured on $100$ pairs of randomly picked nodes, with better results   highlighted in the table; $T_P$ and $T$ represent the graph learning time (without considering the initial kNN graph construction time) and total runtime of the corresponding graph learning algorithms, respectively. The SF-SGL runtime speedups over SGL are also included in the table. We   observe that the graphs learned by SF-SGL can more accurately preserve the spectral properties of the original graph, while the runtime is much smaller than SGL. For example, SF-SGL successfully generate the learned graph for the ``mat3" graph, while SGL failed   due to the limited memory resources.

{{
\begin{figure}
\centering
  \includegraphics[width=0.790596\linewidth]{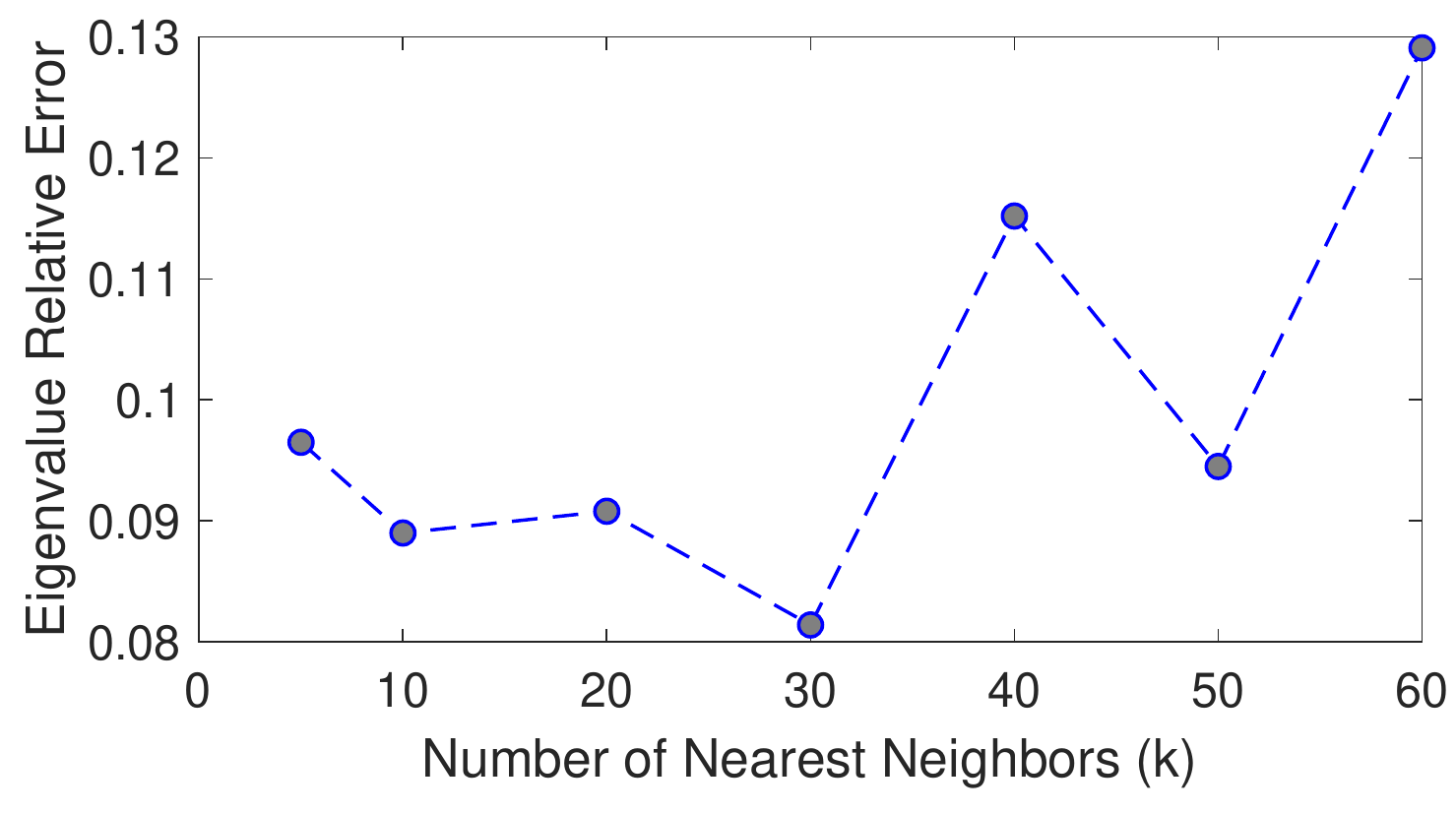}
  {\caption{The average relative errors of the first 50 eigenvalues between the learned graph and the original graph with different $k$ (the ``fe$\_$4elt" graph).}
  \label{fig:eigerr_sfsgl}}
\end{figure}

\begin{figure}
\centering
  \includegraphics[width=0.790596\linewidth]{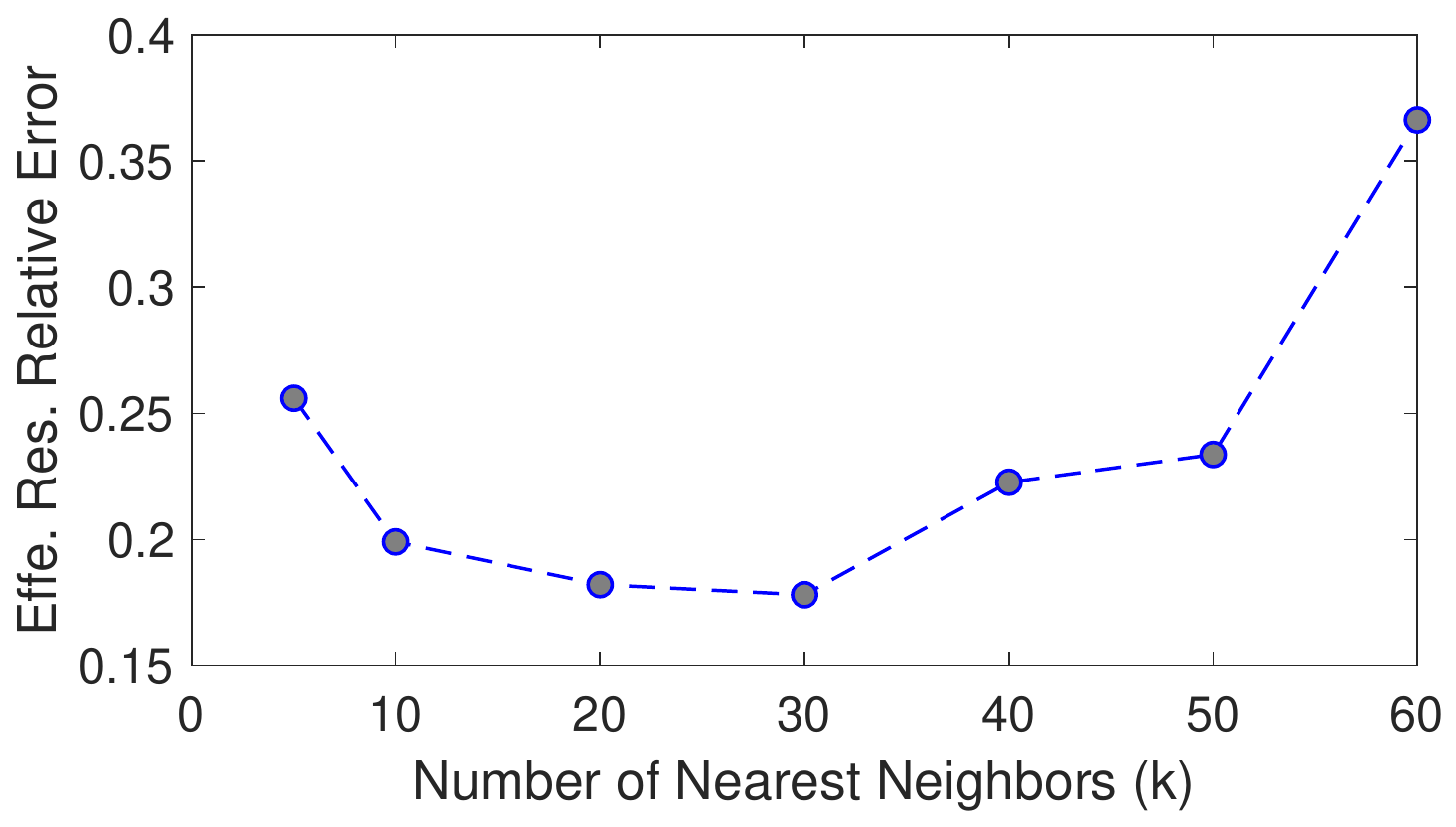}
 {\caption{The average relative errors of the effective resistances (100 pairs) between the learned graph and the original graph with different $k$ (the ``fe$\_$4elt" graph).}
  \label{fig:effeerr_sfsgl}}
\end{figure}
\textbf{Learning with different kNN graphs.}  Given the original ``fe$\_$4elt" graph and the learned graph generated by SF-SGL with a graph density as 1.05, Figure \ref{fig:eigerr_sfsgl} shows the average relative errors of the first $50$ eigenvalues between learned graph and original graph with various number of the nearest neighbors $k$. While Figure \ref{fig:effeerr_sfsgl} shows the average relative errors of the effective resistance values for $100$ randomly picked node pairs  on the learned graph and the original graph with various number of the nearest neighbors $k$. It can be observed that the learned graphs can well preserve the spectral properties of the original graph with various number of the nearest neighbors for initial kNN graphs. 
}}

\begin{figure}
\centering
  \includegraphics[width=0.799596\linewidth]{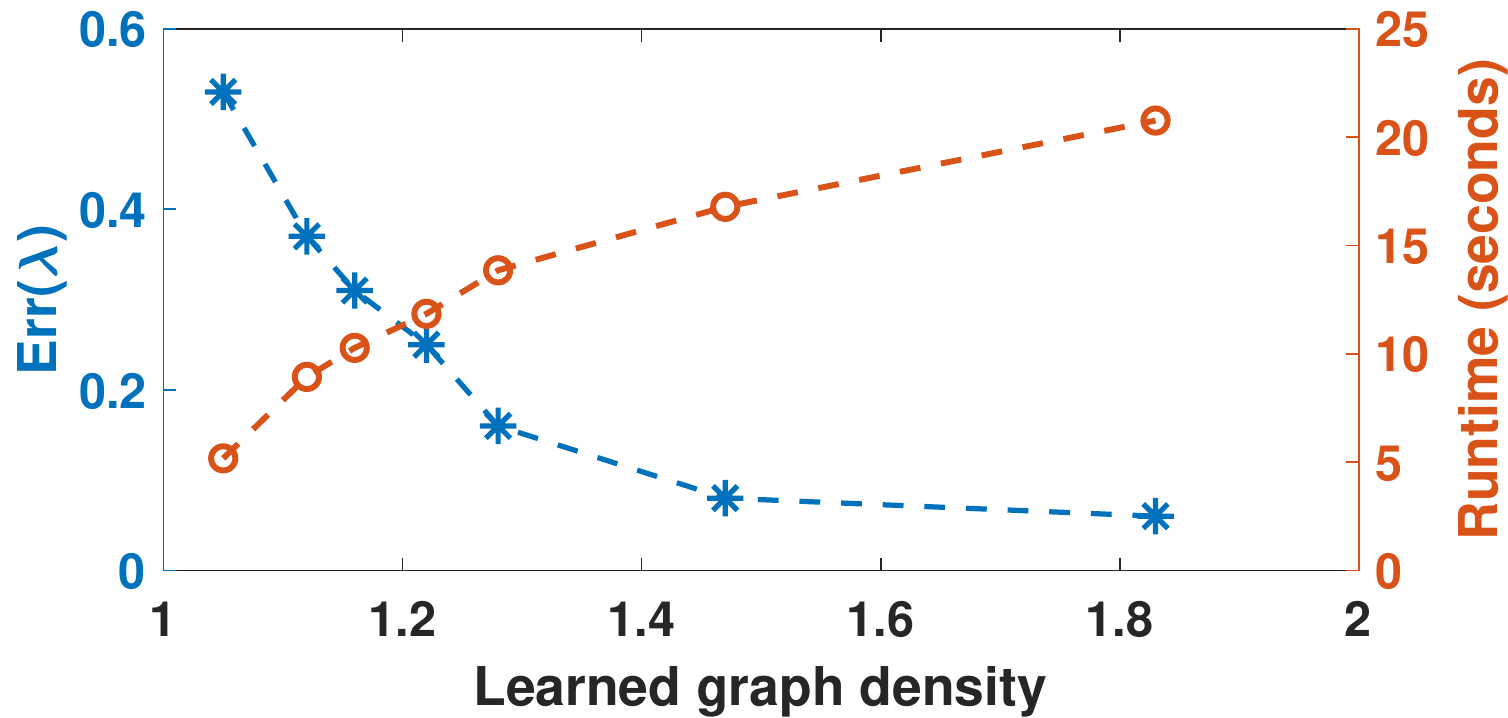}
  \caption{Learned graphs with various graph densities (the ``mat1" graph).}
  \label{fig:mat1}
\end{figure}

\textbf{Trade-off analysis.} Figure \ref{fig:mat1} shows the tradeoffs between the  eigenvalue average relative errors/runtime   and   graph densities learned by SF-SGL for the ``mat1" graph. It can be shown that although higher graph density will result in higher runtime, it will achieve better spectrum preservation of the original graph.

\begin{figure}
  \centering
  \includegraphics[width=0.7995995\linewidth]{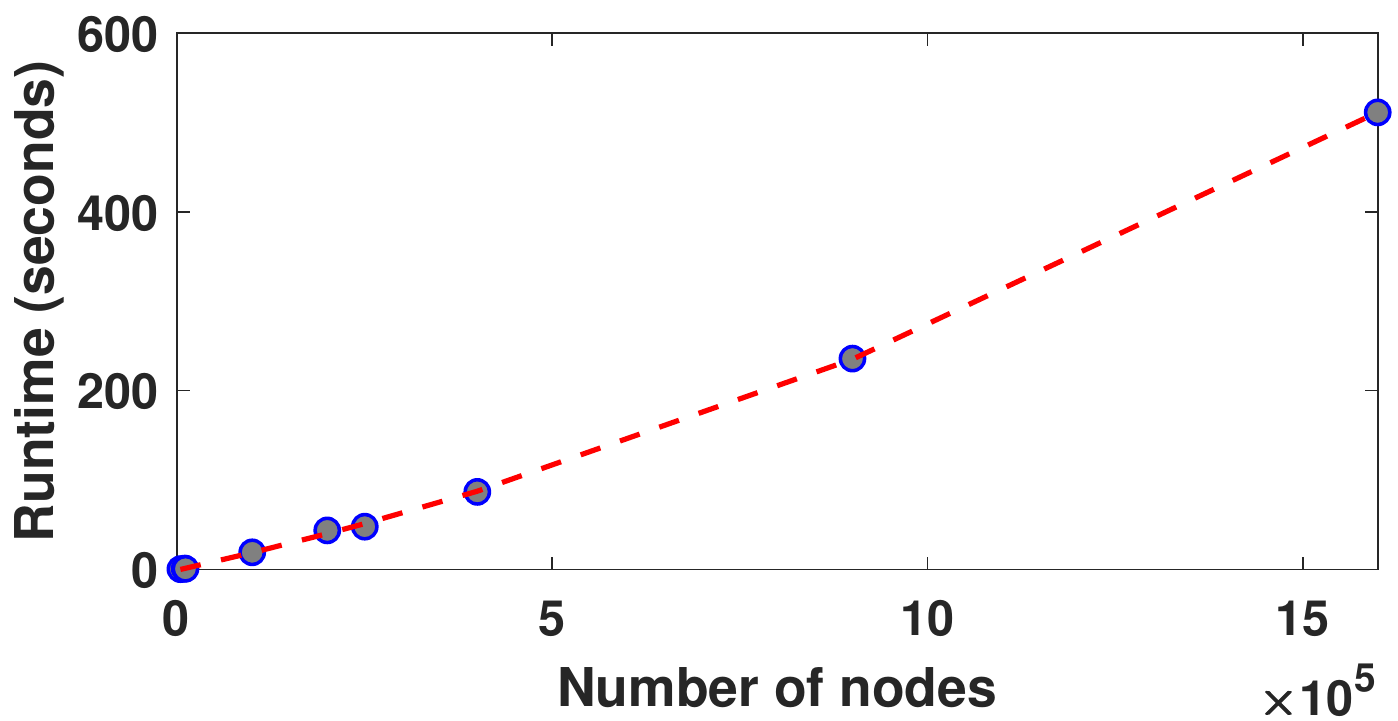}
  \caption{The runtime scalability of the SF-SGL algorithm.}
  \label{fig:runtime}
\end{figure}

\textbf{Runtime scalability.} Figure \ref{fig:runtime} shows the nearly-linear runtime scalability of the proposed SF-SGL algorithm  even for  very large graphs.

\subsection{Data-Driven Vectorless  Integrity Verification}\label{subsection:application}

\begin{table*}
 \centering \caption{Application of SF-SGL in Vectorless Integrity Verification Tasks. }
 \begin{tabular}{|c|c|c|c|c|c|c|c|c|c|c|c|c|}
  \hline
 \multirow{2}{*}{Test cases} & {\multirow{2}{*}{$|V|$}} & {\multirow{2}{*}{$|E|$}} &\multicolumn{4}{|c|}{Original graph} & \multicolumn{6}{|c|}{Learned graph}\\
\cline{4-13}
 { } & { } & { } & {$\tau_G$} & {$T(chol)$}& {$T(sol)$} & {$T(lp)$} & {$\tau_P$} & {$T$} & {$T(chol)$}& {$T(sol)$} & {$T(lp)$} & {$Err$}\\
 \hline  

 {ibm4} & {$9.5E5$} & {$1.6E6$} & {$1.63$} & {$11.86s$} & {$52.82s$} & {$171.60s$} & {$1.02$} & {$237.94s$} & {$1.19s$} & {$3.00s$} & {$162.57s$} & {$0.01\%$} \\
   \hline 
 {thu1} & {$5.0E6$} & {$8.2E6$} & {$1.66$} & {$134.91s$} & {$618.52s$} & {$1,469.92s$} & {$1.03$} & {$1,338.41s$} & {$7.70s$} & {$16.03s$} & {$1,025.85s$} & {$0.01\%$} \\
  \hline 
 {mat1} & {$1.0E5$} & {$2.8E5$} & {$2.88$} & {$3.75s$} & {$15.66s$} & {$14.07s$} & {$1.09$} & {$19.09s$} & {$0.25s$} & {$0.56s$} & {$12.62s$} & {$0.84\%$} \\
  \hline  
{mat2} & {$2.0E5$} & {$5.8E5$} & {$2.93$} & {$16.60s$} & {$56.77s$} & {$23.82s$} & {$1.07$} & {$43.63s$} & {$0.59s$} & {$1.15s$} & {$23.07s$} & {$0.99\%$} \\
  \hline 
{mat3} & {$4.0E5$} & {$1.2E6$} & {$2.89$} & {$27.29s$} & {$89.04s$} & {$69.21s$} & {$1.07$} & {$90.19s$} & {$1.21s$} & {$2.15s$} & {$79.90s$} & {$0.45\%$} \\
  \hline 
{mat4} & {$9.0E5$} & {$2.6E6$} & {$2.89$} & {$169.40s$} & {$333.96s$} & {$531.38s$} & {$1.04$} & {$235.79s$} & {$3.28s$} & {$6.91s$} & {$196.98s$} & {$0.10\%$} \\
  \hline 
{mat5} & {$1.6E6$} & {$4.6E6$} & {$2.90$} & {$1,031.42 s$} & {$38,062.90s$} & {$1623.62s$} & {$1.04$} & {$511.16s$} & {$9.01s$} & {$9.99s$} & {$379.10s$} & {$0.46\%$} \\

  \hline\end{tabular}\label{table:sf_ver}
 \end{table*}
 The integrity of power distribution networks must be verified throughout the design process to ensure that the supply voltage fluctuations are within certain thresholds. To achieve the desired levels of chip reliability and functionality, compute-intensive full-chip thermal analysis and integrity verifications are indispensable, which typically involves estimating thermal profiles under a variety of workloads and power budgets. In this work, we introduce {a  data-driven vectorless power/thermal integrity verification framework}: \textbf{(1)} given a collection   of  voltage/temperature measurements that can be potentially obtained from on-chip voltage/temperature sensors \cite{anik2020chip,ku2019voltage}, the proposed data-driven method will first construct a sparse power/thermal grid network  leveraging the proposed graph topology learning approach; \textbf{(2)} next, vectorless power/thermal integrity verification framework will be exploited  for estimating the worst-case voltage/temperature (gradient) distributions \cite{zhao2020spectral,zhao2017spectral}.

  Table \ref{table:sf_ver} shows the single-level vectorless integrity verification results between learned graphs and original graphs, where ``ibm4" and ``thu1" are power grid test cases, ``mat1" to ``mat5" are 3D thermal grids; $T$, {$T(chol)$}, {$T(sol)$} and  {$T(lp)$} represent the  runtime for SF-SGL, Cholesky factorization, adjoint sensitivity calculation using matrix factors and the linear programming (LP) computation, respectively. {$T(sol)$} and  {$T(lp)$}  are runtime results computed by summing up the runtime for verifying $100$ randomly chosen nodes. $Err$ denotes the average relative solution error compared to the ones obtained with the original graph. Compared with the verification results on original graphs,  the SF-SGL learned graphs  have achieved the very similar results with much lower overall runtime.

\begin{figure}
\centering
  \includegraphics[width=\linewidth]{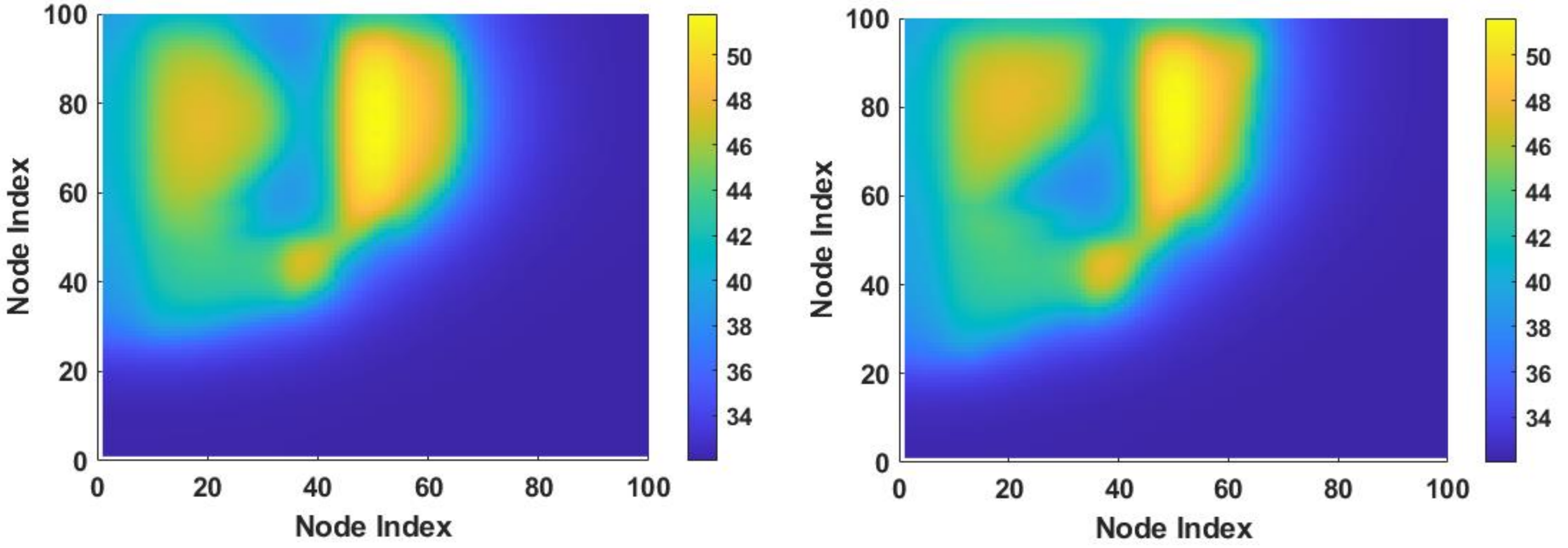}
  \caption{Worst-case temperature distributions obtained by vectorless verification with the original thermal grid (left) and the learned  grid (right).}
  \label{fig:thermal}
\end{figure}

Figure \ref{fig:thermal} shows the worst-case temperature distributions obtained by vectorless verification using the original thermal grid (left) and the learned  grid (right) for the ``mat1" graph. It can be observed that very similar worst-case thermal profiles can be obtained using the SF-SGL learned graph.
   
  
\section{Conclusions}\label{conclusion}
 This work proposes  highly-scalable spectral algorithms for learning large resistor networks from linear voltage and current measurements.  
  We show that the proposed graph learning approach is equivalent to solving the classical graphical Lasso problems with Laplacian-like precision matrices. A    unique feature of the proposed methods is that the learned graphs will have the  {spectral embedding or effective-resistance  distances  to encode the similarities} between the original input data points (voltage measurements). 
As an important extension, we also introduce a more scalable solver-free  spectral graph learning (SF-SGL) algorithm. Such a scalable framework allows learning a hierarchy of spectrally-reduced and sparsified graphs in nearly-linear time, which can become key to accelerating many graph-based numerical computing tasks. The proposed spectral approach is simple to implement and inherently parallel friendly. The proposed spectral algorithms allow  each iteration to be completed in $O(N \log N)$ time, whereas existing state-of-the-art methods   require at least $O(N^2)$ time for each iteration. We also provide  a sample complexity analysis to  show that  it is possible to   accurately recover a resistor network with only $O(\log N)$ voltage measurements (vectors). 
 Our extensive experimental results show that the proposed method can produce a hierarchy of high-quality learned graphs in nearly-linear time  for a variety of real-world, large-scale graphs and circuit networks when compared with prior state-of-the-art spectral methods.
  \vspace{-0pt}



\section{Acknowledgments}
This work is supported in part by  the National Science Foundation under Grants    CCF-2021309,  CCF-2011412, CCF-2212370, and CCF-2205572.

 \bibliographystyle{unsrt}
{
\bibliography{zhang, feng, graphzoom, graspel,zhao}

}

\end{document}